\documentclass[twoside,leqno,onecolumn]{article}

\usepackage[letterpaper]{geometry}

\usepackage{ltexpprt}
\usepackage{hyperref}
\usepackage{url}

\makeatletter
\newcommand*{\addFileDependency}[1]{
  \typeout{(#1)}
  \@addtofilelist{#1}
  \IfFileExists{#1}{}{\typeout{No file #1.}}
}
\makeatother

\usepackage{longtable}
\usepackage{multirow}
\usepackage{booktabs}
\usepackage{algorithm}
\usepackage{algorithmic}
\renewcommand{\algorithmicrequire}{ \textbf{Input:}} 
\renewcommand{\algorithmicensure}{ \textbf{Output:}} 
\usepackage{graphicx}
\usepackage{bbm} 
\usepackage{amsmath} 
\usepackage{amssymb} 
\usepackage{hyperref}
\usepackage{cleveref} 
\allowdisplaybreaks[4] 
\newtheorem{definition}{Definition}[section]
\newtheorem{remark}[Definition]{Remark}
\newtheorem{problem}[Definition]{Problem}
\usepackage{verbatim} 
\usepackage{color}
\usepackage{thm-restate} 

\newcommand{\wxBlue}[1]{{\color{blue}{#1}}}

\newcommand{\wxWhite}[1]{{\color{white}{#1}}}

\newcommand{\wxbigBracket}[1]{\left\{#1\right\}}
\newcommand{\wxsmallBracket}[1]{\left(#1\right)}
\newcommand{\wxmiddleBracket}[1]{\left[#1\right]}
\newcommand{\wxlOneNorm}[1]{\Vert#1\Vert_1}

\newcommand{\wxcalAOT}{\mathcal{AOT}}
\newcommand{\wxcalAWB}{\mathcal{AWB}}
\newcommand{\wxcalRWDz}{\mathcal{W}_z}
\newcommand{\wxcalRWB}{\mathcal{WB}}

\newcommand{\wxcalD}{\mathcal{D}}
\newcommand{\wxcalF}{\mathcal{F}}
\newcommand{\wxcalS}{\mathcal{S}}
\newcommand{\wxcalO}{\mathcal{O}}
\newcommand{\wxcalOTilde}{\widetilde{\mathcal{O}}}
\newcommand{\wxcalX}{\mathcal{X}}
\newcommand{\wxcalPX}{\mathcal{P(X)}}

\newcommand{\wxcalMXplus}{\mathcal{M_+(X)}}

\newcommand{\wxcalWz}{\mathcal{W}_z}
\newcommand{\wxcalWzz}{\mathcal{W}_z^z}
\newcommand{\wxcalQ}{\mathcal{Q}}

\newcommand{\wxbfaout}{\mathbf{a}_{\mathtt{out}}}
\newcommand{\wxaout}{a_{\mathtt{out}}}
\newcommand{\wxbfbout}{\mathbf{b}_{\mathtt{out}}}
\newcommand{\wxbout}{b_{\mathtt{out}}}
\newcommand{\wxmuout}{\mu_{\mathtt{out}}}
\newcommand{\wxnuout}{\nu_{\mathtt{out}}}

\newcommand{\wxbAug}{b_{\mathtt{aug}}}
\newcommand{\wxmuAug}{\mu_{\mathtt{aug}}}
\newcommand{\wxnuAug}{\nu_{\mathtt{aug}}}
\newcommand{\wxbfaAug}{\mathbf{a}_{\mathtt{aug}}}

\newcommand{\wxbfbAug}{\mathbf{b}_{\mathtt{aug}}}

\newcommand{\wxCAug}{C_{\mathtt{aug}}}
\newcommand{\wxbfCAug}{\mathbf{C}_{\mathtt{aug}}}
\newcommand{\wxPAug}{P_{\mathtt{aug}}}
\newcommand{\wxbfPAug}{\mathbf{P}_{\mathtt{aug}}}

\newcommand{\wxdist}{\mathrm{dist}}

\newcommand{\wxVP}{\mathtt{o}}
\newcommand{\wxindicatorFun}{\mathbbm{1}}

\newcommand{\wxbbRplus}{\mathbb{R}_+}
\newcommand{\wxbbE}{\mathbb{E}}
\newcommand{\wxbbR}{\mathbb{R}}
\newcommand{\wxbfOne}{\mathbf{1}}

\newcommand{\wxfixRWB}{{\small \textsf{fixed-RWB}}}
\newcommand{\wxfreeRWB}{{\small \textsf{free-RWB}}}

\newcommand{\wxfixAWB}{{\small \textsf{fixed-AWB}}}
\newcommand{\wxfreeAWB}{{\small \textsf{free-AWB}}}

\newcommand{\wxnuTilde}{\tilde{\nu}}

\newcommand{\wxbfC}{\mathbf{C}}
\newcommand{\wxbfP}{\mathbf{P}}
\newcommand{\wxbfa}{\mathbf{a}}
\newcommand{\wxbfb}{\mathbf{b}}

\newcommand{\wxdiag}{\mathrm{diag}}

\newcommand{\wxVert}[1]{\left\lvert #1 \right\rvert}

\newcommand{\wxOPT}{\mathrm{OPT}}

\newcommand{\wxRefEq}[1]{(\ref{#1})}
\newcommand{\wxRefProm}[1]{\text{Problem} \ref{#1}}

\newcommand{\wxAddErr}{\epsilon_+}

\newcommand{\wxBall}{\mathbb{B}}

\newcommand{\wxDdim}{\textsf{ddim}}

\newcommand{\wxComment}[1]{{\color{blue}{$\triangleright$ #1}}}

\begin{document}

\newcommand\relatedversion{}
\renewcommand\relatedversion{\thanks{The full version of the paper can be accessed at \protect\url{https://arxiv.org/abs/1902.09310}}} 

\title{\Large On Robust Wasserstein Barycenter: The Model and Algorithm}

\author{Xu Wang\thanks{School of Computer Science and Technology, University of Science and Technology of China, \ worm,hjw0330,yangqingyuan,zjp123\}@mail.ustc.edu.cn } \and Jiawei Huang\footnotemark[1] \and Qingyuan Yang\footnotemark[1] \and Jinpeng Zhang \footnotemark[1]}


\date{}

\maketitle







\begin{abstract} \small\baselineskip=11pt 
    
The \emph{Wasserstein barycenter problem} is to compute the average of $m$ given probability measures, which has been widely studied in many different areas;
however, real-world data sets are often noisy and huge, which impedes its application in practice. 
Hence, in this paper, we focus on improving the computational efficiency of two types of \emph{robust Wasserstein barycenter problem} (RWB): \emph{fixed-support} RWB (\wxfixRWB) and \emph{free-support} RWB (\wxfreeRWB); actually, the former is a subroutine of the latter. 
Firstly, we improve efficiency through model reduction; we reduce RWB as an \emph{augmented Wasserstein barycenter problem}, which works for both \wxfixRWB \ and \wxfreeRWB. 
Especially, \wxfixRWB \ can be computed within $\wxcalOTilde(\frac{mn^2}{\wxAddErr})$ time by using an off-the-shelf solver, where $\wxAddErr$ is the pre-specified additive error and $n$ is the size of locations of input measures. 
Then, for \wxfreeRWB, we leverage a quality guaranteed data compression technique, coreset, to accelerate computation by
reducing the data set size $m$. It shows that running algorithms
on the coreset is enough instead of on the original data set. 
Next, by combining the model reduction and coreset techniques above, we propose an algorithm for \wxfreeRWB \ by updating the weights and locations alternatively. 
Finally, our experiments demonstrate the efficiency of our techniques.

\end{abstract}

\section{Introduction}
\label{Sec:intr}

{\em Wasserstein distance} \cite{villani2009optimal} is a popular metric
for quantifying the difference between two probability distributions by the cost of mass transporting. {\em Wasserstein barycenter problem } \cite{agueh2011barycenters} is to compute a measure $\nu$ that minimizes the sum of Wasserstein distances to the given $m$ probability measures $\wxbigBracket{\mu^1,\ldots,\mu^m} $.
Wasserstein barycenter is a popular model for integrating multi-source information, 
where it can efficiently capture the geometric properties of the given data
~\cite{simon2020barycenters, rabin2011wasserstein}.

However, real-world data sets can be noisy and contain outliers. 
The outliers can be located in space arbitrarily; thus, the solution of Wasserstein barycenter problem could be destroyed significantly even with a single outlier. 
To tackle its sensitivity to outliers, we focus on two types of \emph{robust Wasserstein barycenter problem} (RWB)\footnote{Throughout this paper, the robustness refers to the outlier robustness.
}:\emph{ fixed-support} RWB (\wxfixRWB) and \emph{free-support} RWB (\wxfreeRWB). For the former, the locations of barycenter are pre-specified, so we only need to optimize the weights. For the latter, we usually solve it by updating the weights and locations alternatively.

Indeed, the outlier sensitivity is mainly due to its strict marginal constraints.
Thus, to promote the robustness of the fixed-support Wasserstein barycenter,
Le et al. \cite{le2021robust} used a KL-divergence regularization term to alter the hard marginal constraints elegantly.
Nevertheless, only for a particular case $m=2$, 
this formulation can be solved efficiently in $\wxcalOTilde(m n^2/\wxAddErr)$ time,
where $\wxAddErr$ is the pre-specified additive error and $n$ is the size of locations of input measures;
moreover, it suffers from the numerical instability issues caused by the regularization term \cite{peyre2017computational}.
For the \wxfreeRWB, to our best knowledge, few work has been conducted on this.

Besides, the computational burden is also a big obstacle. The free-support Wasserstein barycenter is notoriously hard in computation \cite{altschuler2021hardness}, let alone its robust version. 
Usually, we can accelerate \wxfreeRWB \ by improving the computation at each iteration. Moreover, coreset~\cite{feldman2020core}, a popular data compression technique (defined in \Cref{Def:coreset}), can also accelerate the computation by reducing the data set size $m$.
Roughly speaking, 
a coreset is a small-size proxy of a large-scale input data set $\wxcalQ$ with respect to some objective function while preserving the cost. 
Therefore, one can run an existing optimization algorithm on the coreset $\wxcalS$ rather than on the whole data set $\wxcalQ$, and consequently the total complexity can be reduced.

\paragraph{Contributions:}
We improve the computational efficiency of both \wxfixRWB \ and \wxfreeRWB. In our algorithm, the former \ is a subroutine of the latter (as shown in \Cref{Alg:2}). Our main contributions are as follows:

\begin{itemize}
    \item 
    Firstly, we consider model reduction.
    The method in Nietert et al. \cite{nietert2022outlier} can obtain the value of robust Wasserstein distance efficiently. But it does not compute the coupling (i.e., $\wxbfP$ in \Cref{Def:RWD}), which impedes its application in RWB. 
    Thus, we reduce robust Wasserstein distance problem into an augmented \emph{optimal transport} (OT) problem, which can be solved within $\wxcalOTilde(\frac{n^2}{\wxAddErr})$ time by the off-the-shelf solver \cite{jambulapati2019direct}. Moreover, our method can obtain both value and coupling efficiently. Based on this, we reduce RWB as an \emph{augmented Wasserstein barycenter problem} (AWB); that is, \wxfreeRWB \ and \wxfixRWB \ are reduced as \emph{free-support} AWB (\wxfreeAWB) \ and \emph{fixed-support} AWB (\wxfixAWB) \ respectively. Especially, via this, \wxfixRWB \ can be solved within $\wxcalOTilde(\frac{mn^2}{\wxAddErr})$ time \cite{dvinskikh2021improved} (while Le et al. \cite{le2021robust} can achieve this time complexity only for $m=2$); moreover, the \wxfixAWB \ is an \emph{linear programming} (LP) problem, and thus bypasses the numerical instability issues. Besides, this is essentially a subroutine of \wxfreeRWB, thus helping to improve its computational efficiency at each iteration later.
    
    \item
    Then, for \wxfreeRWB, we accelerate computation by reducing the data set size $m$ via coreset technique.
    We first introduce a method to obtain an $\wxcalO(1)$-approximate solution $\wxnuTilde$ (in \autoref{Lem:approx solution}). Then, we take $\wxnuTilde$ as an anchor to construct an $\wxcalOTilde( \min \{ \wxDdim, \frac{\log n/\epsilon}{\epsilon^2} \} \cdot \frac{n}{\epsilon^2})$ size coreset, and analyze its quality (in \Cref{Them:coreset}). Finally, we propose an algorithm for solving \wxfreeRWB \ by combining model reduction and coreset technique. More specifically, we use the coreset technique to accelerate the computation by reducing the data set size $m$. Meanwhile, to solve \wxfreeRWB, we can compute its corresponding \wxfreeAWB \ by updating the weights and locations alternatively; updating weights is computed by invoking \wxfixAWB \ (as in \Cref{Alg:2}).

\end{itemize}

\subsection{Other Related Works}
\label{Subsec:other related works}

\paragraph{Fixed-support Wasserstein barycenter:}
In recent years, a number of efficient algorithms have been proposed for computing approximation solution with additive error $\wxAddErr$
for fixed-support Wasserstein barycenter.\footnote{Let $\mathsf{opt}$ be the optimal value of fixed-support Wasserstein barycenter problem. The value induced by its approximation solution with additive error $\wxAddErr$ is at most $\mathsf{opt} + \wxAddErr$.} For example,
the Iterative Bregman projections (IBP) algorithm takes $\wxcalOTilde(mn^2/\wxAddErr^2)$ time ~\cite{benamou2015iterative,kroshnin2019complexity};
the accelerated IBP needs $\wxcalOTilde(mn^{5/2}/\wxAddErr)$ time ~\cite{guminov2020accelerated}; the FastIBP requires $\wxcalOTilde(mn^{7/3}/\wxAddErr^{4/3})$ time \cite{lin2020fixed}.
Especially,
Dvinskikh and Tiapkin \cite{dvinskikh2021improved} achieved $\wxcalOTilde(mn^2/\wxAddErr)$ time complexity by using area-convexity property and dual extrapolation techniques.

\paragraph{Coresets:} As a popular data compression technique, coresets have been widely used in machine learning, such as clustering \cite{chen2009coresets,feldman2011unified,braverman2022power}, regression \cite{tukan2020coresets} and robust optimization \cite{huang2022near,wang2021robust}.
Especially, some coreset methods have been proposed to reduce the data set size $m$. For instance,
Ding et al. \cite{ding2018geometric} presented a coreset method for geometric prototype problem; Cohen-Addad et al. \cite{cohen2021improved} designed a coreset for power mean problem; both of them leverage the properties of metric, and can be used to accelerate existing Wasserstein barycenter algorithms. However, due to the fact that robust Wasserstein distance is not a metric, we need to develop new techniques for our \wxfreeRWB.

\section{Preliminaries}
\label{Sec:pre}

\paragraph{Notations:} We define $[n]:= \wxbigBracket{1,\ldots,n}$. We denote the $\ell_1$-norm of a vector by $\wxlOneNorm{\cdot}$.
Let $(\wxcalX,\wxdist)$ be a metric space.
Let $\wxbbRplus$ be the non-negative real number set.
We use $\wxcalMXplus$ to denote the positive measure space on $\wxcalX$, and $\wxcalPX$ the probability measure space.

We use capital boldface letters to denote matrices, such as $\wxbfC$; $C_{ij}$ is its element in the $i$-th row and $j$-th column. Similarly, vector is denoted by lowercase boldface letters, such as $\wxbfa:=(a_1,\ldots,a_d)^T\in\wxbbR^d$; $a_i$ being its $i$-th element; $\wxbfa\succeq \wxbfb$ means $a_i\geq b_i$ for all $i\in[d]$. We define $(\wxbfa;\zeta): = (a_1,\ldots,a_d,\zeta)^T$ for any $\zeta\in\wxbbR$.
We use $\wxdiag(\wxbfa)$ to denote a matrix with diagonal $\wxbfa$ and 0 otherwise.

In the following, we will introduce several fundamental concepts.

\begin{definition}[Wasserstein distance]
    \label{Def:WD}
    Let $\mu = \sum_{i=1}^n a_i\delta_{x_i},\nu = \sum_{j=1}^nb_j\delta_{y_j}\in\wxcalPX$ with $\wxbfa,\wxbfb\in\wxbbRplus^n$ being their weights respectively. Given a real number $z\geq 1$ and a cost matrix $\wxbfC\in\wxbbRplus^{n\times n}$ with $C_{ij} = \wxdist^z(x_i,y_j)$, the $z^{\mathrm{th}}$-Wasserstein distance between $\mu$ and $\nu$ is
    
    \begin{equation*}
        W_z(\mu,\nu) := \wxsmallBracket{ \min_{\wxbfP\in\Pi(\wxbfa,\wxbfb)} \langle \wxbfP,\wxbfC\rangle }^{1/z},
    \end{equation*}     
    where $\Pi(\wxbfa,\wxbfb):= \wxbigBracket{ \wxbfP\in\wxbbRplus^{n\times n} \mid \wxbfP\wxbfOne = \wxbfa, \wxbfP^T\wxbfOne =\wxbfb }$ is the coupling set and $\wxbfOne$ is the vector of ones.
\end{definition}

The Wasserstein distance can capture geometrical structures well, and thus is a popular tool for measuring the similarity between two probability measures. However, due to its sensitivity to outliers, we introduce its robust version to mitigate the impact of outliers. The following robust Wasserstein distance \cite{nietert2022outlier} achieves the minimax optimal robust proxy for $W_p(\cdot,\cdot)$ under Huber contamination \cite{huber1964robust} model. 

\begin{definition}[Robust Wasserstein distance \cite{nietert2022outlier}] \label{Def:RWD}
    Let $\mu,\nu,\wxbfa,\wxbfb,z,\wxbfC$ be the same as in \Cref{Def:WD}, and $\wxmuout,\wxnuout\in \wxcalMXplus$ be the measure of outliers for $\mu,\nu$ respectively. 
    Given $0\leq \zeta_\mu, \zeta_\nu < 1 $ and $\wxbfaout,\wxbfbout$ the weights of $\wxmuout,\wxnuout$, the robust Wasserstein distance $\wxcalRWDz(\mu,\nu)$ between $\mu$ and $\nu$ can be solved by the following optimization problem
    
    \begin{align}\label{Eq:RWD}
            \wxcalRWDz^z(\mu,& \nu)  := \min_{\wxbfaout,\wxbfbout,\wxbfP} g(\wxbfaout,\wxbfbout,\wxbfP)\\
            & g(\wxbfaout,\wxbfbout,\wxbfP) = \langle \wxbfP,\wxbfC\rangle \notag \\ 
            s.t., \quad & \wxbfa \succeq \wxbfaout, \wxlOneNorm{\wxbfaout} = \zeta_\mu, \wxbfb \succeq \wxbfbout, \wxlOneNorm{\wxbfbout} =\zeta_\nu, \notag \\ 
            & \wxbfP\in \Pi\wxsmallBracket{ \frac{\wxbfa-\wxbfaout}{1-\zeta_\mu},\frac{\wxbfb-\wxbfbout}{1-\zeta_\nu}}.\notag
        \end{align}
    
\end{definition}

Note that we use different fonts to distinguish Wasserstein distance $W_z(\cdot,\cdot)$ and its robust version $\wxcalRWDz(\cdot,\cdot)$. In \Cref{Def:RWD}, $\mu,\nu$ contain $\zeta_\mu,\zeta_\nu$ masses of outliers, respectively. 
Moreover, $\wxlOneNorm{\frac{\wxbfa-\wxbfaout}{1-\zeta_\mu}}=1$ and $\wxlOneNorm{\frac{\wxbfb-\wxbfbout}{1-\zeta_\nu}}=1$ holds;
therefore, $\frac{\mu - \wxmuout}{1-\zeta_\mu},\frac{\nu - \wxnuout}{1-\zeta_\nu} \in\wxcalPX$ can be regarded as robust proxies of $\mu,\nu$.

Next, we introduce its corresponding two types of RWB: \wxfixRWB \ and \wxfreeRWB. Now, we pre-specify some notations used throughout this paper.
The input of \wxfixRWB/\wxfreeRWB \ is a set of probability measures,
\begin{equation}\label{Eq:Q}
    \wxcalQ:=\wxbigBracket{\mu^l \mid \mu^l= \sum_{i=1}^n a_i^l\delta_{x_i^l}}_{l\in[m]} \subseteq \wxcalPX,
\end{equation}
where $\omega:\wxcalQ\rightarrow\wxbbRplus$ is the weight function of $\wxcalQ$. We also define $\omega_l:=\omega(\mu^l)$. 
For simplicity, we consider the case that the locations of each measure in $\wxcalQ$ are $n$ points. 
$X^l: = \wxbigBracket{x_i^l}_{i\in[n]}$ and $\wxbfa^l$ are the locations and weights of $\mu^l$.
We use $|\wxcalQ|$ to denote the number of measures in $\wxcalQ$. For any $\wxcalQ'\subseteq\wxcalQ$, we define $\omega(\wxcalQ'):=\sum_{\mu\in\wxcalQ'}\omega(\mu)$. 
Let 
\begin{equation} \label{Eq:nu_tildenu}
    \nu = \sum_{j=1}^nb_j \delta_{y_j}, \wxnuTilde = \sum_{j=1}^n \tilde{b}_j \delta_{\tilde{y}_j} \in\wxcalPX
\end{equation}
be two feasible solutions. 
$Y:=\wxbigBracket{y_j}_{j\in[n]}$ and $\wxbfb$ are the locations and weights of 
$\nu$; similarly, $\widetilde{Y}:=\wxbigBracket{\tilde{y}_j}_{j\in[n]}$ and $\mathbf{\tilde{b}}$ are the locations and weights of $\wxnuTilde$.

\begin{definition}(RWB)
    Given a set of probability measures $\wxcalQ$ as in \wxRefEq{Eq:Q},
    the robust Wasserstein barycenter on
    $\wxcalQ$ is a new probability measure $\nu^*\in\wxcalPX$ that minimizes the following objective
    \begin{equation} \label{Eq:RWB}
        \wxcalRWB(\wxcalQ,\nu):=    \frac{1}{\omega(\wxcalQ)} \sum_{l=1}^m \omega_l \cdot \wxcalRWDz^z(\mu^l,\nu) \quad \text{ with } \nu\in\wxcalPX.
    \end{equation}
    We call it \wxfixRWB \ if the locations of $\nu
    $ are pre-specified; we call it \wxfreeRWB \ if the locations of $\nu$ can be optimized.
\end{definition}

We define $\wxOPT := \min_{\nu} \wxcalRWB(\wxcalQ,\nu)$ as the optimal value induced by an optimal solution $\nu^*$. 
A probability measure $\wxnuTilde\in\wxcalPX$ is called an \emph{$\alpha$-approximate solution} if it achieves at most $\alpha\cdot \wxOPT$.

Most optimization problems employ iterative algorithms; thus, the solution is generally limited to a local area after the initial rounds; for such scenarios, a local coreset \cite{wang2021robust,huang2021novel} is enough.
 
\begin{definition}[Local coreset] \label{Def:coreset}
    Let $\wxcalQ$ be the same as in \wxRefEq{Eq:Q} and $0<\epsilon<1$. We fix an anchor $\wxnuTilde$ as in \wxRefEq{Eq:nu_tildenu}.
    The local $\epsilon$-coreset $\wxcalS$ is a subset of $\wxcalQ$ with $\tau$ being its weight function and satisfies that, for all $\nu\in\wxcalD_{\wxnuTilde}$,
    \begin{equation}\label{Eq:coreset property}
        \wxVert{\wxcalRWB(\wxcalQ,\nu) - \wxcalRWB(\wxcalS,\nu)} \leq \epsilon \cdot \wxcalRWB(\wxcalQ,\nu), 
    \end{equation}
    where we define $\wxcalRWB(\wxcalS,\nu): = \frac{1}{\tau(\wxcalS)}\sum_{\mu\in\wxcalS}\tau(\mu)\cdot \wxcalWzz(\mu,\nu) $ and
 $\wxcalD_{\wxnuTilde} := \wxbigBracket{ \nu\in\wxcalPX \mid W_z(\wxnuTilde,\nu) \leq r}$.
\end{definition}

The ``local'' coreset means that \wxRefEq{Eq:coreset property} holds for all solutions in the local area $\wxcalD_{\wxnuTilde}$. The local area $\wxcalD_{\wxnuTilde}$ is anchored at the pre-specified solution $\wxnuTilde$; specifically, the solution $\nu$ in $\wxcalD_{\wxnuTilde}$ satisfies that the Wasserstein distance between $\nu,\wxnuTilde$ is no larger than $r$.

\paragraph{Organization:} The remainder of this paper is as follows. \Cref{Sec:model} focuses on model reduction for two types of RWB: \wxfreeRWB \ and \wxfixRWB.
\Cref{Sec:free} focuses on improving time efficiency for \wxfreeRWB; \Cref{Subsec:coreset,Subsec:analysis} introduce coreset technique to reduce the data set size $m$ and analyze its quality; \Cref{Subsec:alg} proposes an algorithm  to solve \wxfreeRWB \ by combining model reduction and coreset technique.

\section{Model Reduction}
\label{Sec:model}
This section improves the computational efficiency from the perspective of model reduction.
\Cref{Subsec:reduce RWD} reduces robust Wasserstein distance problem as augmented OT problem by \Cref{Lem:reduction AOT};
\Cref{Subsec:reduce RWB} reduces \wxfreeRWB \ as \wxfreeAWB \ in \Cref{Them:RWB}; \Cref{Subsec:fixed_RWB} reduces \wxfixRWB \ as \wxfixAWB, which is essentially a special case of \wxfreeRWB.

\subsection{Robust Wasserstein Distance}
\label{Subsec:reduce RWD}

To facilitate model reduction, we specify the following notations. Let $\mu,\nu,\wxbfa,\wxbfb,\zeta_\mu,\zeta_\nu,z$ be the same as in \Cref{Def:RWD} and $\wxVP$ the dummy point. Let $A\subseteq \wxbigBracket{x_i}_{i\in[n+1]}$ and $B\subseteq \wxbigBracket{y_j}_{j\in[n+1]}$ with $x_{n+1} = y_{n+1} = \wxVP$.
We define

    \begin{align}\label{Eq:notation for AOT}
        \wxmuAug(A) := \frac{\mu(A-\wxVP)}{1-\zeta_\mu} + \wxindicatorFun_{\wxVP\in A} \cdot \frac{\zeta_\nu}{1 - \zeta_\nu} \\ \text{ and its weights } \wxbfaAug:=(\frac{\wxbfa}{1-\zeta_\mu};\frac{\zeta_\nu}{1-\zeta_\nu}) \notag \\
        \wxnuAug(B) := \frac{\nu(B-\wxVP)}{1-\zeta_\nu} + \wxindicatorFun_{\wxVP\in B} \cdot \frac{\zeta_\mu}{1-\zeta_\mu} \\ \text{ and its weights } \wxbfbAug:= (\frac{\wxbfb}{1-\zeta_\nu};\frac{\zeta_\mu}{1-\zeta_\mu}) \notag,
    \end{align}
where $\wxindicatorFun_{\text{condition}}$ is $1$ if the condition is true, and $0$ otherwise.

\begin{problem}[Augmented OT] \label{Prom:AOT}
    For any real number $z\geq 1$, we define the augmented OT as the following optimization problem
    \begin{equation}\label{Eq:AOT}
        \begin{aligned}
            \wxcalAOT(\wxmuAug,\wxnuAug) := \min_{\wxbfPAug\in\Pi(\wxbfaAug,\wxbfbAug)} h(\wxbfPAug), \\ h(\wxbfPAug) := \langle \wxbfPAug,\wxbfCAug\rangle \quad \text{s.t.,} \quad (\wxPAug)_{n+1,n+1} = 0, 
        \end{aligned}
    \end{equation}
    where $(\wxCAug)_{ij}:= \wxindicatorFun_{i\in[n] \text{ and } j\in[n]} \cdot \wxdist^z(x_i,y_j)$. 
\end{problem}

\begin{remark}
    We use a similar dummy point idea as in ~\cite{chapel2020partial,ding2023data}.
    $(\wxPAug)_{n+1,n+1} = 0$ always holds for the optimal solution of $h(\cdot)$ in \wxRefEq{Eq:AOT}. 
    We write it explicitly to facilitate the construction of a bijection $\varphi$ later; via $\varphi$, we can regard \wxRefEq{Eq:RWD} and \wxRefEq{Eq:AOT} as equivalence.
\end{remark}

Actually, OT is a generalization of Wasserstein distance. The cost function of Wasserstein distance must be induced by a metric, while OT only requires the cost function to be a non-negative real value function. Therefore, \wxRefProm{Prom:AOT} is indeed an OT problem.

Next, we illustrate that \wxRefEq{Eq:RWD} is equivalent to \wxRefEq{Eq:AOT} under a map $\varphi$. Let $\wxcalD_g$ and $\wxcalD_h$ be the feasible domain of \wxRefEq{Eq:RWD} and \wxRefEq{Eq:AOT} respectively. 
We construct a map $\varphi$ from $\wxcalD_g$ to $\wxcalD_h$ as follows,
\begin{equation}\label{Eq:phi}
    \begin{aligned}
        & \varphi:\wxcalD_g\rightarrow\wxcalD_h, (\wxbfaout,\wxbfbout,\wxbfP ) \mapsto \wxbfPAug,
    \end{aligned} 
\end{equation}
where $(\wxPAug)_{ij} = \wxindicatorFun_{i\in[n]\text{ and } j\in[n]} \cdot P_{ij} + \wxindicatorFun_{i\in[n]\text{ and } j=n+1} \cdot \frac{(\wxaout)_i}{1-\zeta_{\mu}} + \wxindicatorFun_{i=n+1\text{ and } j\in[n]} \cdot \frac{(\wxbout)_j}{1-\zeta_{\nu}} $.

\begin{restatable}{lem}{LemEquiv} \label{Lem:reduction AOT}
    \wxRefEq{Eq:RWD} is equivalent to \wxRefEq{Eq:AOT} under bijection $\varphi$; that is, $h = g\varphi^{-1}$ and $h\varphi = g$. More specifically, for any $\wxbfPAug\in\wxcalD_h$, $h(\wxbfPAug) = g\varphi^{-1}(\wxbfPAug)$ holds; for any $(\wxbfaout,\wxbfbout,\wxbfP)\in\wxcalD_g$, $h\varphi(\wxbfaout,\wxbfbout,\wxbfP ) = g(\wxbfaout,\wxbfbout,\wxbfP)$ holds.
\end{restatable}
The proof of \Cref{Lem:reduction AOT} is deferred  to Appendix.

\begin{remark}
    Removing the last condition in $\wxRefEq{Eq:AOT}$ does not affect the result; thus, it is a classic OT problem.
Nietert et al. \cite{nietert2022outlier} computed the value of robust Wasserstein distance efficiently by its dual problem; however, it does not offer coupling, which impedes its applications to RWB.
    Our method can obtain coupling within $\wxcalOTilde(\frac{n^2}{\wxAddErr})$ time efficiently \cite{jambulapati2019direct}, which makes it possible to apply this technique to solve RWB. \\
\end{remark}

\subsection{Free-support Robust Wasserstein Barycenter}
\label{Subsec:reduce RWB}

This section reduces \wxfreeRWB \ as \wxfreeAWB. 
\textbf{Note that}, henceforth, we consider the case that, for $\wxcalWz(\cdot,\cdot)$ and $\wxcalRWB(\cdot,\cdot)$, the measure of the first parameter contains $\zeta$ mass of outliers; and the measure of the second parameter contains no outliers.

\begin{problem}(\wxfreeAWB)
    The free-support augmented Wasserstein barycenter on
    $\wxcalQ$ is a new probability measure $\nu^*\in\wxcalPX$ that minimizes the following objective
    
    \begin{equation}\label{Eq:AWB}
        \begin{aligned}
            \wxcalAWB(\wxcalQ,\nu) := \frac{1}{\omega(\wxcalQ)}\sum_{l=1}^m \omega_l \cdot \wxcalAOT(\wxmuAug^l,\wxnuAug)
        \end{aligned}
    \end{equation}
    for $\nu\in\wxcalPX$, where $\wxmuAug^l(A) = \frac{\mu^l(A-\wxVP)}{1-\zeta}$ for $\mu^l\in\wxcalQ$    and $\wxnuAug(B) := \nu(B-\wxVP) + \wxindicatorFun_{\wxVP\in B} \cdot \frac{\zeta}{1-\zeta}$.
    
\end{problem}

Also, the first parameter of both $\wxcalRWB(\cdot,\cdot),\wxcalAWB(\cdot,\cdot)$ is an input measure set, and the second parameter is called feasible solution.

\begin{restatable}{thm}{ThemRWB}\label{Them:RWB}
 For any feasible solution $\nu\in\wxcalPX$, \wxRefEq{Eq:RWB} and \wxRefEq{Eq:AWB} induce the same cost; more specifically, $\wxcalRWB(\wxcalQ,\nu) = \wxcalAWB(\wxcalQ,\nu)$ for all $\nu\in\wxcalPX$.
\end{restatable}

The proof of \Cref{Them:RWB} is in Appendix.

By \Cref{Them:RWB}, we can solve \wxfreeRWB \ by computing \wxfreeAWB. 
Moreover, it can be extended to the case that barycenter contains $\zeta_\nu$ mass of outliers and $\mu^l$ contains $\zeta_\mu$ mass of outliers.
Besides, both \Cref{Lem:reduction AOT} and \Cref{Them:RWB} work for arbitrary probability measures, such as continuous, semi-continuous, discrete measures~\cite{peyre2017computational}.

\subsection{ Fixed-support Robust Wasserstein Barycenter}
\label{Subsec:fixed_RWB}
Actually, \wxfixRWB/\wxfixAWB \ is a special case of \wxfreeRWB/\wxfreeAWB \ by pre-specifying the locations of barycenter. 
We fix the locations $Y$ of $\nu$ as in \wxRefEq{Eq:nu_tildenu}. Then, the \wxfixRWB \ on $\wxcalQ$ is to find a measure $\nu\in\wxcalPX$ located on $Y$ that minimizes

\begin{equation}
\begin{aligned}\label{Eq:fixed-RWB}
& \wxcalRWB_Y(\wxcalQ,\nu) := \frac{1}{\omega(\wxcalQ)} \sum_{l=1}^m \omega_l \cdot \langle \wxbfP^l,\wxbfC^l \rangle \\
s.t. \quad & \wxbfP^l \in\wxbbRplus^{n\times n}, \wxbfOne^T\wxbfP^l\wxbfOne = 1, \wxbfP^l\wxbfOne \preceq \frac{\wxbfa^l}{1-\zeta}\text{ for } l\in[m] \\
& (\wxbfP^l)^T\wxbfOne = (\wxbfP^{l+1})^T\wxbfOne \text{ for } l\in[m-1],
\end{aligned}
\end{equation}
where $\wxbfC^l\in\wxbbRplus^{n\times n}$ is the cost matrix between $\mu^l$ and $\nu$ with $C_{ij}^l = \wxdist^z(x_i^l,y_j)$.

By using \Cref{Them:RWB}, the \wxfixRWB \ is equivalent to the following \wxfixAWB
\begin{equation}\label{Eq:fixed_AWB}
    \begin{aligned}
 \wxcalAWB_Y(\wxcalQ,\nu) := \frac{1}{\omega(\wxcalQ)}\sum_{l=1}^m \omega_l \langle \wxbfPAug^l,\wxbfCAug^l \rangle \\
 s.t. \quad \wxbfPAug^l\in \Pi(\wxbfaAug^l,\wxbfbAug), \quad (\wxbAug)_{n+1} = \frac{\zeta}{1-\zeta},
    \end{aligned}
\end{equation}
where $\wxbfCAug^l\in\wxbbRplus^{n\times(n+1)}$ is the cost matrix between $\wxmuAug^l$ and $\wxnuAug$ with $(\wxCAug^l)_{ij} = \wxindicatorFun_{i\in[n] \text{ and } j\in[n]} \cdot \wxdist^z(x_i^l,y_j)$.

\begin{remark}
    For \wxfixRWB, the locations of $\nu$ are pre-specified; thus, $\wxbfCAug^l$ is constant. Obviously, \wxfixRWB \ is a linear programming (LP) problem, which can be solved within $\wxcalOTilde(\frac{mn^2}{\wxAddErr})$ time \cite{dvinskikh2021improved}.    Other algorithms for fixed-support Wasserstein barycenter in \Cref{Subsec:other related works} can also be used to solve \wxfixRWB. \\
\end{remark}

\begin{remark}
Our \wxfixRWB \ has the following advantages:
(\romannumeral1) The method \ in \cite{le2021robust} can achieve $\wxcalOTilde(\frac{mn^2}{\wxAddErr})$ time complexity only for $m=2$; however, our method can achieve it for all $m \geq 2$.
(\romannumeral2) The method \cite{le2021robust} achieves robustness by adding an elegant KL-divergence regularization term to its objective; however, this term makes it suffer from numerical instability \cite{peyre2017computational}; while our \wxfixRWB \ is an LP essentially, and thus bypasses this problem. \\
\end{remark}

\section{Algorithm for Free-support Robust Wasserstein Barycenter}
\label{Sec:free}

This section focuses on improving the computational efficiency of \wxfreeRWB. \Cref{Subsec:coreset} constructs a local coreset and \Cref{Subsec:analysis} analyzes its quality.
\Cref{Subsec:alg} proposes an algorithm for solving \wxfreeRWB, which leverages coreset technique and model reduction.

\subsection{Algorithm for Constructing Coresets}
\label{Subsec:coreset}

This section uses coreset technique to reduce the data set size $m$ for \wxfreeRWB.
Now, we provide a method to obtain an approximate solution for \wxfreeRWB,
which is an appropriate ``anchor'' for constructing a local coreset later.

\begin{restatable}{lem}{LemApproxSolution}\label{Lem:approx solution}
    Let $\alpha>1$. 
    If we select $t$ measures $\wxbigBracket{ \mu^{q1},\ldots,\mu^{qt} }$ from $\wxcalQ$ according to the distribution $\frac{\omega_l}{\omega(\wxcalQ)}$. Let $X^{qi}$ be the locations of $\mu^{qi}$ and 
         \begin{equation*} \label{Eq:approx solution of free_RWB}
            \wxnuTilde = \arg\min_{\nu} \wxbigBracket{\wxcalRWB_X(\wxcalQ,\nu) \mid X\in\wxbigBracket{X^{q1},\ldots,X^{qt}}}.
        \end{equation*}
Then, $\wxnuTilde$ yields a $2^z\alpha$-approximate solution for \wxfreeRWB \ on $\wxcalQ$ with probability at least $1-\alpha^{-t}$.    
If $\alpha = \wxcalO(1)$, we can obtain an $\wxcalO(1)$-approximate solution.
\end{restatable}
The proof of \Cref{Lem:approx solution} is deferred to Appendix.

\paragraph{Coreset construction:} Let $\wxcalQ$ be the same as in \wxRefEq{Eq:Q} and $K = \lceil \log\frac{1}{\epsilon} \rceil$.
Let $\wxnuTilde = \sum_{j=1}^n\tilde{b}_j\delta_{\tilde{y}_j}\in\wxcalPX$ be an $\wxcalO(1)$-approximate solution of \wxfreeRWB \ on $\wxcalQ$ and $H = \sqrt[z]{\wxcalRWB(\wxcalQ,\wxnuTilde)}$. Then, by using anchor $\wxnuTilde$, we partition $\wxcalQ$ into $K+2$ layers\footnote{The notation $\sqcup$ denotes disjoint union of sets.}, i.e., $\wxcalQ = \sqcup_{k=0}^{K+1}\wxcalQ_k$,

\begin{equation}\label{Eq:partition layers}
\begin{aligned}
   & \wxcalQ_0 = \wxbigBracket{ \mu\in\wxcalQ \mid \wxcalRWDz(\mu,\wxnuTilde) \leq H } \\
   & \wxcalQ_k = \wxbigBracket{ \mu\in\wxcalQ \mid 2^{k-1}H < \wxcalRWDz(\mu,\wxnuTilde) \leq 2^{k} H } \text{ for } k\in[K]\\
   & \wxcalQ_{K+1} = \wxbigBracket{ \mu\in\wxcalQ \mid \wxcalRWDz(\mu,\wxnuTilde) > 2^k H }.    
\end{aligned}
\end{equation}

Then, for any $\wxcalQ_k$ with $0\leq k\leq K$, if $|\wxcalQ_k| \leq \Gamma$, we take $\wxcalQ_k$ as $\wxcalS_k$, and keep their original weights constant; if $|\wxcalQ_k| > \Gamma$, we take $\Gamma$ samples $\wxcalS_k$ from $\wxcalQ_k$ according to distribution $\wxindicatorFun_{\mu^l\in\wxcalQ_k}\cdot\frac{\omega_l}{\omega(\wxcalQ_k)}$, and set $\tau(\mu) = \frac{\omega(\wxcalQ_k)}{\Gamma}$ for $\mu\in\wxcalS_k$.

Intuitively, the measure $\mu$ in the outermost layer $\wxcalQ_{K+1}$ is too far from the local area $\wxcalD_{\wxnuTilde}$. For a fixed $\mu\in\wxcalQ_{K+1}$, the costs $\wxcalRWDz(\mu,\nu)$ induced by all solutions in $\wxcalD_{\wxnuTilde}$ are similar; thus, we can sample less points.
For the outermost layer $\wxcalQ_{K+1}$, we set  $\wxcalS_{K+1} = \{ \mu^{\max},\mu^{\min} \}$, where 
\begin{equation}\label{Eq:muMaxMin}
 \begin{aligned}
 \mu^{\max}: & = \arg\max_{\mu\in\wxcalQ_{K+1}} \wxcalRWDz(\mu,\wxnuTilde) \\
 \mu^{\min}: & = \arg\min_{\mu\in\wxcalQ_{K+1}} \wxcalRWDz(\mu,\wxnuTilde).
 \end{aligned}
\end{equation}
We set their weight as $\tau(\mu^{\max})$ and $\tau(\mu^{\min})$ satisfying 
\begin{equation}\label{Eq:muMaxMin_weights}
\begin{aligned}
 \tau(\mu^{\max}) \cdot \wxcalRWDz^z(\mu^{\max},\wxnuTilde) + \tau(\mu^{\min}) \cdot \wxcalRWDz^z(\mu^{\min},\wxnuTilde) \\
 = \sum_{\mu\in\wxcalQ_{K+1}} \omega(\mu)\cdot\wxcalRWDz^z(\mu,\wxnuTilde).
\end{aligned}
\end{equation}

Then, we put all $\wxcalS_k$ together to obtain $\wxcalS$. The details for constructing coresets are shown in \Cref{Alg:1}.

\begin{algorithm}[h]
    \caption{Coreset for \wxfreeRWB}
 \label{Alg:1}
    \algorithmicrequire Probability measure set $\wxcalQ$,\\
    \wxWhite{111111} $\wxcalO(1)$-approximate solution $\wxnuTilde$. \\
    \vspace{-4mm}
    \begin{algorithmic}[1]
        \STATE Let $K = \lceil \log\frac{1}{\epsilon} \rceil$ and $H = \sqrt[z]{\wxcalRWB(\wxcalQ,\wxnuTilde)}$;
        \STATE Partition $\wxcalQ$ into $K+2$ layers $\wxbigBracket{\wxcalQ_0,\wxcalQ_1,\ldots,\wxcalQ_{K+1}}$ as in \wxRefEq{Eq:partition layers};
        \FOR{each $\wxcalQ_k$ with $0\leq k \leq K$}
        \IF{$|\wxcalQ_k| \leq \Gamma$}
 \STATE \wxComment{$\Gamma$ will be specified in \Cref{Them:coreset}}; \\
        \STATE Set $\wxcalS_k = \wxcalQ_k$ and set $\tau(\mu) = \omega(\mu)$ for $\mu\in\wxcalQ_k$;
        \ELSE
        \STATE 
 Take samples $\wxcalS_k$ with $|\wxcalS_k| = \Gamma$ according to the distribution $\wxindicatorFun_{\mu^l\in\wxcalQ_k}\cdot\frac{\omega_l}{\omega(\wxcalQ_k)}$, and set $\tau(\mu) = \frac{\omega(\wxcalQ_k)}{\Gamma}$ for $\mu\in\wxcalS_k$;
        \ENDIF
        \ENDFOR
 \STATE We set $\wxcalS_{K+1} = \wxbigBracket{\mu^{\min}, \mu^{\max}}$, and set their weights as in \wxRefEq{Eq:muMaxMin} and \wxRefEq{Eq:muMaxMin_weights};
 \\
 \STATE $\wxcalS = \cup_{k=0}^{K+1}\wxcalS_k$;
    \end{algorithmic}
    \algorithmicensure the coreset $\wxcalS$ for $\wxcalQ$.
\end{algorithm}

\begin{remark}
 Our coreset method is inspired by the layered sampling in \cite{chen2009coresets,braverman2022power}, which is analyzed on metric space. However, our robust Wasserstein distance is not a metric, and thus new techniques are needed to analyze it.
\end{remark}

\subsection{Theoretical Analysis}
\label{Subsec:analysis}

In this section, we prove the quality guarantee of our \Cref{Alg:1} in \Cref{Them:coreset}.

\begin{restatable}{thm}{ThemCoreset} 
\label{Them:coreset}
We set $r=H$ in \Cref{Def:coreset}.
Suppose the diameter of $\wxcalX$ is $R$, i.e., $\max_{x,y\in\wxcalX} \wxdist(x,y) = R$ and the metric space has doubling dimension $\wxDdim$. Let $\wxnuTilde$ be an $\wxcalO(1)$-approximate solution of \wxfreeRWB \ on $\wxcalQ$ and $\Gamma = \wxcalOTilde( \min \{ \wxDdim, \frac{\log n/\epsilon}{\epsilon^2} \} \cdot \frac{n}{\epsilon^2})$. With probability at least $1-\eta$, \Cref{Alg:1} outputs a local coreset $\wxcalS$ for \wxfreeRWB \ on $\wxcalQ$.

\end{restatable}

\begin{remark}
 (\romannumeral1) The doubling dimension\footnote{The doubling dimension is defined in Appendix.} is a measure for describing the growth rate of the data set $\wxcalX$ with respect to the metric $\wxdist(\cdot,\cdot)$;
 (\romannumeral2) The coreset size in \Cref{Them:coreset} is dependent of the dataset set size $m$;
 (\romannumeral3) If $\nu$ is an $\alpha$-approximate solution on $\wxcalS$, then it is also a $\frac{1+\epsilon}{1-\epsilon}\alpha$-approximate solution on $\wxcalQ$ with probability at least $1-\eta$; thus, the approximate solution on $\wxcalS$ can imitate the approximate solution on $\wxcalQ$ well.
 
\end{remark}

From the iterative size reduction \cite{narayanan2019optimal} and the terminal version of Johnson-Lindenstrauss Lemma \cite{braverman2021coresets}, we can regard $\wxDdim = \wxcalOTilde(\frac{\log n/\epsilon}{\epsilon^2}$). Therefore, it is sufficient to obtain a coreset with size $\wxcalOTilde(\frac{n \cdot \wxDdim}{\epsilon^2})$.

\begin{restatable}{lem}{boundByr}\label{Lem:bounded by r}
    Given $r,\wxnuTilde,\wxcalD_{\wxnuTilde}$ as in \Cref{Def:coreset} and $0< s\leq 1$,    we have $\wxcalRWDz^z(\mu,\nu) \leq (1+s)^{z-1} \cdot \wxcalRWDz^z(\mu,\wxnuTilde) + (1+\frac{1}{s})^{z-1} \cdot r^z$ for any $\mu\in\wxcalQ, \nu\in\wxcalD_{\wxnuTilde}$.
\end{restatable}

\Cref{Lem:bounded by r} indicates that, for any fixed $\mu$, the robust Wasserstein distance between $\mu$ and all the solutions in $\wxcalD_{\wxnuTilde}$ can be bounded. Based on this, we can obtain the following lemma. (The proofs of \Cref{Lem:bounded by r} and \Cref{Lem:fix a solution} are in Appendix.)

\begin{restatable}{lem}{LemLocalCoreset} \label{Lem:fix a solution}
 We set $r=H$ in \Cref{Def:coreset}.    Let $\wxnuTilde$ be an $\wxcalO(1)$-approximate solution of \wxfreeRWB \ on $\wxcalQ$.
 For a fixed solution $\nu\in\wxcalD_{\wxnuTilde}$, if we set $\Gamma = \wxcalO(\frac{\log 1/\eta}{\epsilon^2})$ in \Cref{Alg:1}, with probability at least $1-\eta$, we have 
 \begin{equation*}
 \wxVert{\wxcalRWB(\wxcalQ,\nu) - \wxcalRWB(\wxcalS,\nu)} \leq \wxcalO(\epsilon) \cdot \wxcalRWB(\wxcalQ,\nu).
 \end{equation*}
\end{restatable}

 However, \Cref{Lem:fix a solution} only works for a single solution $\nu\in\wxcalD_{\wxnuTilde}$. To ensure \wxRefEq{Eq:coreset property} holds for all
 the solutions in $\wxcalD_{\wxnuTilde}$, we take two steps:(\expandafter{\romannumeral1})
discrete the solution space $\wxcalD_{\wxnuTilde}$ by grid $\wxcalF$ (defined in \wxRefEq{Eq:grid}), and make sure it holds for all solutions in $\wxcalF$; (\expandafter{\romannumeral2}) bound the discretization error and make sure it holds for all $\nu\in\wxcalD_{\wxnuTilde}$. (The above two steps are exactly the roadmap for proving \Cref{Them:coreset}.)

\begin{proof}
\textbf{Discrete the solution space $\wxcalD_{\wxnuTilde}$:}
 Suppose $R$ is the diameter of $\wxcalX$.
 We define $\wxBall(x,R): = \wxbigBracket{ y\mid \wxdist(x,y) \leq R }$.

Let $\wxBall_j^{\epsilon H}$ be an $\epsilon H$-net of the ball $\wxBall(\tilde{y}_j,R)$. 
Since the metric space $(\wxcalX,\wxdist)$ has doubling dimension $\wxDdim$, we have $|\wxBall_j^{\epsilon H}| = (\frac{R}{\epsilon H})^{\wxcalO(\wxDdim)}$. 
Let $e = \lceil\frac{nR^z}{\epsilon^z H^z}\rceil$ and $E = \wxbigBracket{ \frac{i}{e} \mid i\in [e] }$ be a $\frac{\epsilon^z H^z }{n R^z}$-net of the interval $[0,1]$.
Then, we consider a grid $\wxcalF$ as follows.

\begin{equation}\label{Eq:grid}
    \wxcalF := \wxbigBracket{ \bar{\nu} = \sum_{j=1}^n\bar{b}_j\delta_{\bar{y}_j}\in\wxcalPX \mid \bar{b}_j\in E, \bar{y}_j\in\wxBall_j^{\epsilon H} } 
\end{equation}
        
    Thus, we obtain $|\wxcalF|\leq (\frac{R}{\epsilon H})^{\wxcalO(n\cdot\wxDdim)}\cdot (\frac{nR}{\epsilon H})^{\wxcalO(n)}$. 
 By using union bound for $|\wxcalF|$ solutions, it yields that \wxRefEq{Eq:coreset property} holds for all the solutions in $\wxcalF$ with probability at least $1-\eta$ if we set $\Gamma = \wxcalO( \frac{n \log 1/\eta}{\epsilon^2} (\wxDdim \cdot \log\frac{R}{\epsilon H} + \log n) )$. \\

\textbf{Bound discretization error:}
For any $\nu = \sum_{j=1}^n b_j\delta_{y_j}\in\wxcalD_{\wxnuTilde}$, we can find a $\bar{\nu}' = \sum_{j=1}^n b_j\delta_{\bar{y}_j}$ satisfying $\wxdist(y_j,\bar{y}_j)\leq \wxcalO(\epsilon H)$. For $W_z(\nu,\bar{\nu}' )$, we can construct a feasible coupling by assigning all the mass of $y_j$ to $\bar{y}_j$ for all $j\in[n]$; that is, the coupling is a diagonal matrix $\wxdiag(\wxbfb)$.
 Then, the cost induced by the feasible coupling $\wxdiag(\wxbfb)$ is at most $\wxcalO(\epsilon^z H^z)$. Since $W_z^z(\nu,\bar{\nu}')$ is induced by the optimal coupling, we have $W_z^z(\nu,\bar{\nu}') \leq \wxcalO(\epsilon^z H^z)$.

Meanwhile, for $\bar{\nu}'$, we can find a $\bar{\nu} = \sum_{j=1}^n\bar{b}_j\delta_{\bar{y}_j} \in\wxcalF$ satisfying $b_j - \bar{b}_j \leq \frac{\epsilon^z H^z}{nR^z}$. 
For $W_z(\bar{\nu},\bar{\nu}')$, we can construct a feasible coupling by keeping $b_j - \frac{\lfloor b_j \cdot e \rfloor}{e}$ mass constant; thus, we need to assign at most $\wxcalO(\frac{\epsilon^z H^z}{R^z})$ mass in total; obviously, the cost caused by this feasible coupling is at most $\wxcalO(\frac{\epsilon^z H^z}{R^z})\cdot R^z$. Therefore, we have $W_z^z(\bar{\nu},\bar{\nu}') \leq \epsilon^z H^z$.
Finally, for any $\nu\in\wxcalD_{\wxnuTilde}$, we can find a grid point $\bar{\nu}$ satisfying $W_z^z(\bar{\nu},\nu) \leq \wxcalO(\epsilon^z)\cdot H^z$.

Similar to \Cref{Lem:bounded by r}, we have $\wxcalRWDz^z(\mu,\nu) \leq (1+\epsilon)^{z-1}\cdot\wxcalRWDz^z(\mu,\bar{\nu}) + (1+\frac{1}{\epsilon})^{z-1} \cdot W_z^z(\nu,\bar{\nu})$
and 
$\wxcalRWDz^z(\mu,\nu) \geq (1-\wxcalO(\epsilon)) \cdot \wxcalRWDz^z(\mu,\bar{\nu}) - \wxcalO(\frac{1}{\epsilon})^{z-1} \cdot W_z^z(\nu,\bar{\nu})$. Then, 

 \begin{align*}
 & \wxVert{ \wxcalRWB(\wxcalQ,\nu) - \wxcalRWB(\wxcalS,\nu) } \\
 = & \frac{1}{\omega(\wxcalQ)} \wxVert{ \sum_{\mu\in\wxcalQ} \omega(\mu) \cdot \wxcalRWDz^z(\mu,\nu) - \sum_{\mu\in\wxcalS} \tau(\mu) \cdot \wxcalRWDz^z(\mu,\nu) } \\
 \leq & \frac{1}{\omega(\wxcalQ)} \vert \sum_{\mu\in\wxcalQ} \omega(\mu) \cdot (1+\wxcalO(\epsilon)) \cdot \wxcalRWDz^z(\mu,\bar{\nu}) - \\
  & \sum_{\mu\in\wxcalS} \tau(\mu) \cdot (1-\wxcalO(\epsilon)) \cdot \wxcalRWDz^z(\mu,\bar{\nu}) \vert + \wxcalO(\frac{1}{\epsilon^{z-1}}) \cdot W_z^z(\nu,\bar{\nu}) \\
 = & \wxVert{ (1+\wxcalO(\epsilon)) \wxcalRWB(\wxcalQ,\bar{\nu}) - (1-\wxcalO(\epsilon)) \wxcalRWB(\wxcalS,\bar{\nu}) } + \wxcalO(\epsilon H^z)\\
 \leq & \wxcalO(\epsilon) \cdot \wxcalRWB(\wxcalQ,\bar{\nu}) + \wxcalO(\epsilon H^z) = \wxcalO(\epsilon) \cdot \wxcalRWB(\wxcalQ,\bar{\nu}),
 \end{align*}
where the first equality and the second equality come from the definition of $\wxcalRWB(\cdot,\cdot)$, the first inequality follows from generalized triangle inequality, and the last equality is due to the fact that $H^z = \wxcalO(1) \cdot \wxOPT \leq \wxcalO(1) \cdot \wxcalRWB(\wxcalQ,\bar{\nu})$. 
By using \Cref{Lem:bounded by r} again, we have $\wxcalRWB(\wxcalQ,\bar{\nu}) \leq (1+\wxcalO(\epsilon))\cdot \wxcalRWB(\wxcalQ,\nu)) + \wxcalO(\epsilon H^z)$; then, we obtain $\wxVert{ \wxcalRWB(\wxcalQ,\nu) - \wxcalRWB(\wxcalS,\nu) } \leq \wxcalO(\epsilon) \cdot \wxcalRWB(\wxcalQ,\nu)$.
\end{proof}

\subsection{Algorithm}
\label{Subsec:alg}
\Cref{Alg:2} aims to compute a solution for \wxfreeRWB \ on the input probability measure set $\wxcalQ$ (as in $\wxRefEq{Eq:Q}$), and
it leverages model reduction and coreset technique to accelerate the computation.

By model reduction, to solve \wxfreeRWB, we can compute \wxfreeAWB \ by updating the weights and locations alternatively. We pre-specify the number of iterations as $\mathsf{Iter}$.
At each iteration, updating weights is finished by invoking \wxfixAWB.
Updating locations is a power mean problem; for a special case $z=2$, it is a geometric mean problem \cite{cohen2021improved}.

By coreset technique, we can construct a small proxy $\wxcalS$ (returned by \Cref{Alg:1}) anchored at a $\wxcalO(1)$-approximate solution. For the $q$-th iteration, we assume 
 $\wxcalS = \wxbigBracket{\mu^{ql}}_{l\in[|\wxcalS|]}$,
where $X^{ql}:=\{x_i^{ql}\}_{i\in[n]}$ is the locations of $\mu^{ql}$. 
We can run algorithms on $\wxcalS$ instead of on the original data set $\wxcalQ$, which can be very large. Thus, it can improve efficiency by reducing the data set size $m$. 
Besides, since coreset $\wxcalS$ only works for local area $\wxcalD_{\wxnuTilde}$, we reconstruct it if $\nu'$ is out of $\wxcalD_{\wxnuTilde}$.

\begin{algorithm}[h]
    \caption{Algorithm for \wxfreeRWB}
 \label{Alg:2}
    \algorithmicrequire Probability measure set $\wxcalQ$, \\ \wxWhite{111111} $\wxcalO(1)$-approximate solution $\wxnuTilde$.\\
    \vspace{-4mm}
    \begin{algorithmic}[1]
 \STATE $\nu = \wxnuTilde$, $q = 0$;
        \STATE Construct a local coreset $\wxcalS$ anchored at $\nu$ for \wxfreeRWB \ on $\wxcalQ$ by \Cref{Alg:1}; \label{Alg2:line01}
        \FOR{$q < \mathsf{Iter}$}
 \STATE $q = q + 1$;
        \STATE     Compute the coupling set $\wxbigBracket{\wxbfPAug^l}_{l\in[m]}$ of the \wxfixAWB \ on $\wxcalQ$ by invoking \wxRefEq{Eq:fixed_AWB}; \label{Alg2:line03}
        \STATE $\wxbfb' = (\wxbfPAug^1)^T \wxbfOne \in\wxbbRplus^{n+1}$;
        
        \STATE $\nu' = \sum_{j=1}^n b_j' \delta_{y_j}$;
        \IF{ $\nu'\in\wxcalD_{\wxnuTilde}$ }
        \STATE $b_j = b_j'$ for $j\in[n]$; \ \wxComment{update the weights of $\nu$};
        \ELSE
        \STATE $\nu = \sum_{j=1}^nb_j\delta_{y_j}$ and turn to Line \ref{Alg2:line01}; \label{Alg2:line05}
        \ENDIF
        \STATE Compute $y_j' = \arg\min_{y} \sum_{l=1}^{|\wxcalS|} \tau(\mu^{ql}) \sum_{i=1}^n P_{ij}^{ql} \cdot\wxdist^z(x_i^{ql},y)$ for all $j\in[n]$; \label{Alg2:line07}
 \STATE $\nu' = \sum_{j=1}^n b_j \delta_{y_j'}$;
        \IF{ $\nu'\in\wxcalD_{\wxnuTilde}$ }
        \STATE $y_j = y_j'$ for $j\in[n]$;  \wxComment{update the locations of $\nu$};
        \ELSE
        \STATE $\nu = \sum_{j=1}^nb_j\delta_{y_j}$ and turn to Line \ref{Alg2:line01}; \label{Alg2:line09}
        \ENDIF
    \ENDFOR
    \end{algorithmic}
    \algorithmicensure $\nu = \sum_{j=1}^nb_j\delta_{y_j}$.
\end{algorithm}

\begin{remark}
 Actually, \Cref{Alg:2} offers a framework for  solving \wxfreeRWB. We can use the techniques in Frank–Wolfe algorithm \cite{jaggi2013revisiting,jorge2006numerical} to improve it in the future.
\end{remark}

\begin{table*}[h]
\caption{Comparisons of our \wxfreeRWB \ and the original WB algorithm under different noise intensity. We use $\zeta$ to denote the total mass of outliers, and the noise distribution (N.D.) is Gaussian distribution $\mathcal{N}(\cdot,\cdot)$.}
\label{Tab:free}
\begin{tabular}{llllllll}
 \toprule
 \multirow{2}*{$\zeta$} & \multirow{2}*{N.D.} & \multicolumn{3}{c}{Our \wxfreeRWB} & \multicolumn{2}{c}{WB} & \\
 \cmidrule(lr){3-5}\cmidrule(lr){6-8}
 & & runtime ($\downarrow$) & \textsf{WD} ($\downarrow$) & \textsf{cost} ($\downarrow$) & runtime ($\downarrow$) & \textsf{WD} ($\downarrow$) & \textsf{cost} ($\downarrow$) \\
 \midrule
 \multirow{3}*{$0.1$} & $\mathcal{N}(20,20^2)$ &  $406.26_{\pm 101.29}$ & $\textbf{0.56  }_{\pm 0.02 }$ & $\textbf{57.98 }_{\pm 0.02 }$ & $ 429.72_{\pm 125.86}$ & $ 0.57   _{\pm 0.02 }$ & $  57.99  _{\pm 0.02 }$   \\
                      & $\mathcal{N}(40,40^2)$ &  $374.55_{\pm 111.53}$ & $\textbf{2.86  }_{\pm 0.11 }$ & $\textbf{59.46 }_{\pm 0.11 }$ & $ 371.07_{\pm 109.45}$ & $ 91.27  _{\pm 31.51}$ & $  114.97 _{\pm 27.88}$   \\
                      & $\mathcal{N}(60,60^2)$ &  $299.06_{\pm 29.56 }$ & $\textbf{0.85  }_{\pm 0.04 }$ & $\textbf{58.27 }_{\pm 0.04 }$ & $ 313.12_{\pm 42.11 }$ & $ 473.66 _{\pm 7.86 }$ & $  444.91 _{\pm 6.86 }$   \\ \hline
 \multirow{3}*{$0.2$} & $\mathcal{N}(20,20^2)$ &  $434.46_{\pm 71.85 }$ & $\textbf{2.50  }_{\pm 0.09 }$ & $\textbf{59.47 }_{\pm 0.09 }$ & $ 465.89_{\pm 105.96}$ & $ 2.56   _{\pm 0.09 }$ & $  59.52  _{\pm 0.09 }$   \\
                      & $\mathcal{N}(40,40^2)$ &  $462.63_{\pm 109.26}$ & $\textbf{15.78 }_{\pm 3.13 }$ & $\textbf{59.81 }_{\pm 1.01 }$ & $ 448.58_{\pm 67.21 }$ & $ 382.56 _{\pm 6.81 }$ & $  335.22 _{\pm 6.47 }$   \\
                      & $\mathcal{N}(60,60^2)$ &  $460.43_{\pm 66.94 }$ & $\textbf{3.71  }_{\pm 0.12 }$ & $\textbf{59.22 }_{\pm 0.26 }$ & $ 481.08_{\pm 56.36 }$ & $ 1201.98_{\pm 8.23 }$ & $  1054.11_{\pm 7.66 }$   \\ \hline
 \multirow{3}*{$0.3$} & $\mathcal{N}(20,20^2)$ &  $477.99_{\pm 74.03 }$ & $\textbf{5.84  }_{\pm 0.11 }$ & $\textbf{61.41 }_{\pm 0.20 }$ & $ 509.11_{\pm 48.46 }$ & $ 5.99   _{\pm 0.11 }$ & $  61.51  _{\pm 0.21 }$   \\
                      & $\mathcal{N}(40,40^2)$ &  $306.84_{\pm 39.58 }$ & $\textbf{198.63}_{\pm 74.20}$ & $\textbf{174.58}_{\pm 62.32}$ & $ 312.85_{\pm 35.12 }$ & $ 661.14 _{\pm 7.50 }$ & $  540.72 _{\pm 7.90 }$   \\ 
                      & $\mathcal{N}(60,60^2)$ &  $444.65_{\pm 110.59}$ & $\textbf{10.86 }_{\pm 1.14 }$ & $\textbf{56.03 }_{\pm 0.76 }$ & $ 519.71_{\pm 76.13 }$ & $ 2093.29_{\pm 10.88}$ & $  1811.10_{\pm 9.60 }$   \\
 \bottomrule\end{tabular}
\end{table*}

\begin{figure*}
    \centering
    \includegraphics[width=0.9\linewidth]{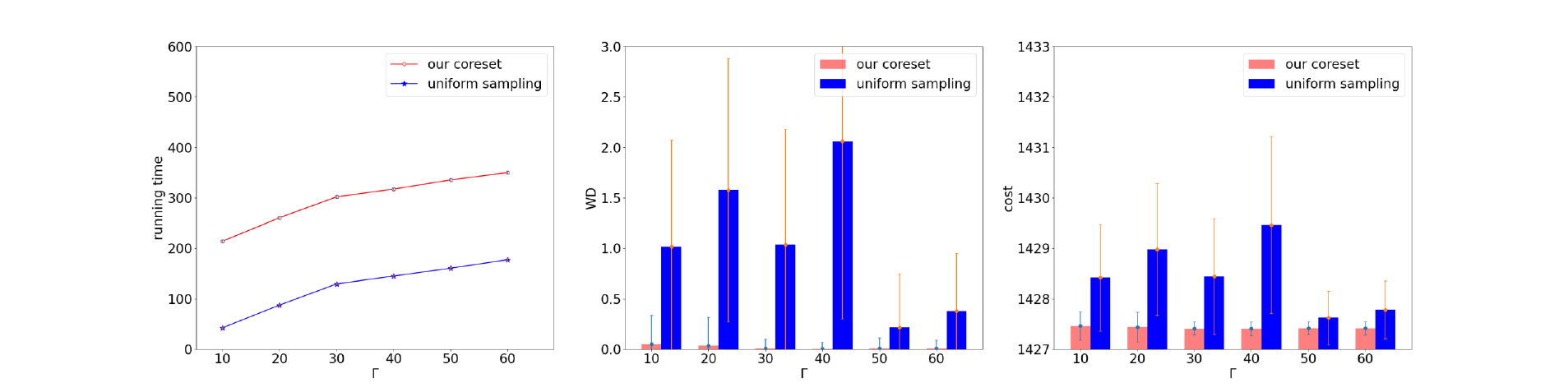}
    \caption{Comparisons of our coreset and uniform sampling. $\Gamma$ is the sample size for each layer in our method.}
\label{Fig:free}
\end{figure*}

\section{Experiments}
\label{Sec:exp}

This section demonstrates the effectiveness of the \wxfreeRWB \ and the efficiency of the coreset technique. 
All the experiments were conducted on a server equipped with 2.40GHz Intel CPU, 128GB main memory, and Python 3.8.
Due to space limitations, we only show part of the experimental results here. (More experiments can be found in Appendix.)

We evaluate our method on the MNIST \cite{lecun1998gradient}dataset,
which is a popular handwritten benchmark with digits from 0 to 9. 
We select $3000$ images. For the $l$-th image, we represent it by a measure $\mu^l = \sum_{i=1}^{60}a_i^l\delta_{x_i^l} \in\wxcalPX$ via $k$-means clustering. 
More specifically, we take $28 \times 28$ pixels as the input of clustering, and obtain 60 cluster centers $X^l = \wxbigBracket{x_i^l}_{i\in[60]}$ as the locations of $\mu^l$; the weight $a_i^l$ is 
proportional to the total pixel number of this cluster. Till now, we obtain a clean data set $\wxcalQ = \wxbigBracket{\mu^l}_{l\in[3000]}$. We assume $\wxcalQ'$ is its corresponding noisy data set, which will be constructed later. 

The method proposed by Le et al. \cite{le2021robust} suffers from numerical instability caused by the KL-divergence regularization term. In almost all the instances here, it failed to produce results.
Thus, we only compare our method with the original Wasserstein barycenter (WB) algorithm here.\footnote{The codes and full experiments (including the results on the other data sets and contrast experiments on the numerical instability issues 
\cite{le2021robust}) are available at  \wxBlue{\url{https://github.com/little-worm/iiccllrr2023/blob/main/Full_experiments.pdf}}.}

To measure the performance of our methods, we consider three criteria: (\romannumeral1) running time: the CPU time consumed by algorithms; (\romannumeral2) \textsf{WD}: the Wasserstein distance between the barycenter computed on the noisy data set $\wxcalQ'$ and the WB on the clean data set $\wxcalQ$; (\romannumeral3) \textsf{cost}: Let $\nu$ be the barycenter computed by some algorithm on the noisy data set $\wxcalQ'$. We define its \textsf{cost} as $WB(\wxcalQ,\nu)$. (Note that its cost is evaluated on the clean data set $\wxcalQ$.)
We run $10$ trials and record the average results.

First, to show the efficiency of our \wxfreeRWB, we add $\zeta$ mass of Gaussian noise to each measure $\mu^l\in\wxcalQ$ to obtain a noisy data set $\wxcalQ'$. Then, we compute barycenter by the original WB algorithm and our \wxfreeRWB \ on the noisy dataset respectively. 
The results in \Cref{Tab:free} show that our \wxfreeRWB \ can tackle outliers effectively under different noise intensity.

Then, we show the efficiency of our coreset technique in \Cref{Fig:free}. In this scenario, to obtain a noisy data set $\wxcalQ'$, we first add $0.1$ mass of Gaussian noise from $\mathcal{N}(40,40)$ to each measure $\mu^l\in\wxcalQ$, and then shift the locations of 300 images randomly according to the distribution $\mathcal{N}(0,80)$.  
Throughout our experiments, we ensure that the total sample size of the uniform sampling method equals to the coreset size of our method.
Our coreset method is much more time consuming. However, it performs well on the criteria \textsf{WD} and \textsf{cost}.
Moreover, our method is more stable.

\section{Conclusion and future work}
\label{Sec:conclusion}

In this paper, we study two types of RWB: \wxfixRWB\ and \wxfreeRWB. Our \wxfixRWB \ can be solved within $\wxcalO(\frac{mn^2}{\wxAddErr})$ time efficiently by reducing it as an LP;
for \wxfreeRWB, we use model reduction and coreset technique to accelerate it. 
Obviously, our \Cref{Alg:2} performs well in practice.
In theory, we can guarantee that the output is a constant approximate solution. In the future, how to compute a $(1+\epsilon)$-approximate solution for $0<\epsilon<1$ is worth studying.

\section*{Acknowledgment}

Thanks to Professor Ding for his help and instruction on this paper.
``Data were provided (in part) by the Human Connectome Project, WU-Minn Consortium (Principal Investigators: David Van Essen and Kamil Ugurbil; 1U54MH091657) funded by the 16 NIH Institutes and Centers that support the NIH Blueprint for Neuroscience Research; and by the McDonnell Center for Systems Neuroscience at Washington University.''

\vspace{9mm}

\appendix

\begin{center}
    {\Large \textbf{Appendix}}
\end{center}

\section{Full Experiments}

Here, we demonstrate the effectiveness of the \wxfixRWB/\wxfreeRWB \ and the efficiency of the coreset technique. 
All the experiments were conducted on a server equipped with 2.40GHz Intel CPU, 128GB main memory, and Python 3.8.

We evaluate our method on three datasets: MNIST \cite{lecun1998gradient}, ModelNet40 \cite{wu20153d} and Human Connectome Project (HCP) \cite{van2013wu}.

\paragraph{MNIST:} MNIST dataset \cite{lecun1998gradient} is a popular handwritten benchmark with digits from 0 to 9. We select $3000$ images. For the $l$-th image, we represent it by a measure $\mu^l = \sum_{i=1}^{60}a_i^l\delta_{x_i^l} \in\wxcalPX$ via $k$-means clustering. 
More specifically, we take $28 \times 28$ pixels as the input of clustering, and obtain 60 cluster centers $X^l = \wxbigBracket{x_i^l}_{i\in[60]}$ as the locations of $\mu^l$; the weight $a_i^l$ is 
proportional to the total pixel number of this cluster. Till now, we obtain a clean dataset $\wxcalQ = \wxbigBracket{\mu^l}_{l\in[3000]}$. We assume $\wxcalQ'$ is its corresponding noisy dataset, which will be constructed for each dataset later.

\paragraph{ModelNet40:} ModelNet40 \cite{wu20153d}  is a comprehensive clean collection of $3D$ CAD models. 
We choose $989$ CAD models of chair.
First, we convert these CAD models into point clouds. Then, each point cloud was grouped into $k = 60$ clusters; each cluster was represented by its cluster center;
the weight of each center is proportional to the total number of points of the cluster.
Then, we can obtain $\wxcalQ$ and $\wxcalQ'$ (as described in MNIST dataset).

\paragraph{Human Connectome Project (HCP):} Human Connectome Project (HCP)
\cite{van2013wu} is a dataset of high-quality neuroimaging data in over 1100 healthy young adults, aged 22–35. We took 3000 3D brain images. For each image, its voxels were  grouped into $k = 60$ clusters; each cluster was represented by its cluster center;
the weight of each center is proportional to the total number of points of the cluster.
Then, we can obtain $\wxcalQ$ and $\wxcalQ'$ (as described in MNIST dataset).

The method proposed by Le et al. \cite{le2021robust} suffers from numerical instability caused by the KL-divergence regularization term. 
(This is illustrated in \Cref{App:numerical}. )
In almost all the instances here, it failed to produce results.
Thus, we only compare our method with the original Wasserstein barycenter (WB) algorithm here.\footnote{The codes and full experiments (including the results on the other datasets and contrast experiments on the numerical instability issues 
\cite{le2021robust}) are available at  \wxBlue{\url{https://github.com/little-worm/iiccllrr2023/blob/main/Full_experiments.pdf}}.}

To measure the performance of our methods, we consider three criteria: (\romannumeral1) running time: the CPU time consumed by algorithms; (\romannumeral2) \textsf{WD}: the Wasserstein distance between the barycenter computed on the noisy dataset $\wxcalQ'$ and the WB on the clean dataset $\wxcalQ$; (\romannumeral3) \textsf{cost}: Let $\nu$ be the barycenter computed by some algorithm on the noisy dataset $\wxcalQ'$. We define its \textsf{cost} as $WB(\wxcalQ,\nu)$. (Note that its cost is evaluated on the clean dataset $\wxcalQ$.)
We run $10$ trials and record the average results.

\subsection{Experiments on MNIST Dataset}
\label{App:experiments}

This section shows the experimental result on MNIST dataset.
First, to show the efficiency of our \wxfixRWB/\wxfreeRWB, we add $\zeta$ mass of Gaussian noise to each measure $\mu^l\in\wxcalQ$ to obtain a noisy dataset $\wxcalQ'$. Then, we compute barycenter by the original fixed-support/free-support WB algorithm and our \wxfixRWB/\wxfreeRWB \ on the noisy dataset respectively. 
The results in \Cref{Tab:fixed}/\Cref{Tab:free} show that our \wxfixRWB/\wxfreeRWB \ can tackle outliers effectively under different noise intensity.

\begin{table}[htbp]
\caption{Comparisons of our \wxfixRWB \ and the original fixed-support WB algorithm under different noise intensity on MNIST. We use $\zeta$ to denote the total mass of outliers, and the noise distribution (N.D.) is Gaussian distribution $\mathcal{N}(\cdot,\cdot)$.}
\label{Tab:fixed}
\begin{tabular}{llllllll}
 \toprule
 \multirow{2}*{$\zeta$} & \multirow{2}*{N.D.} & \multicolumn{3}{c}{Our \wxfixRWB} & \multicolumn{2}{c}{WB} & \\
 \cmidrule(lr){3-5}\cmidrule(lr){6-8}
 & & runtime ($\downarrow$) & \textsf{WD} ($\downarrow$) & \textsf{cost} ($\downarrow$) & runtime ($\downarrow$) & \textsf{WD} ($\downarrow$) & \textsf{cost} ($\downarrow$) \\
 \midrule
 \multirow{3}*{$0.1$} & $\mathcal{N}(20,20^2)$ &  $89.38 _{\pm 24.89}$   & $\textbf{0.14  }_{\pm 0.00}$  & $ \textbf{3.54  }_{\pm 0.00}$ & $ 99.11 _{\pm 31.32}$  & $ \textbf{0.14}  _{\pm 0.00}$  & $ \textbf{3.54}  _{\pm 0.00}$   \\
                      & $\mathcal{N}(40,40^2)$ &  $80.83 _{\pm 23.43}$   & $\textbf{13.73 }_{\pm 0.10}$  & $ \textbf{23.06 }_{\pm 0.11}$ & $ 74.42 _{\pm 13.06}$  & $ 21.92 _{\pm 0.23}$  & $ 32.25 _{\pm 0.26}$   \\
                      & $\mathcal{N}(60,60^2)$ &  $66.78 _{\pm 14.84}$   & $\textbf{88.49 }_{\pm 1.51}$  & $ \textbf{86.17 }_{\pm 1.32}$ & $ 75.15 _{\pm 10.62}$  & $ 189.53_{\pm 1.36}$  & $ 175.33_{\pm 1.20}$   \\ \hline
 \multirow{3}*{$0.2$} & $\mathcal{N}(20,20^2)$ &  $67.96 _{\pm 24.50}$   & $\textbf{0.29  }_{\pm 0.01}$  & $ \textbf{4.52  }_{\pm 0.00}$ & $ 96.60 _{\pm 24.53}$  & $ 0.30  _{\pm 0.01}$  & $ 4.53  _{\pm 0.00}$   \\
                      & $\mathcal{N}(40,40^2)$ &  $100.53_{\pm 25.94}$   & $\textbf{25.42 }_{\pm 0.19}$  & $ \textbf{38.18 }_{\pm 0.22}$ & $ 119.50_{\pm 30.14}$  & $ 44.04 _{\pm 0.24}$  & $ 60.19 _{\pm 0.29}$   \\
                      & $\mathcal{N}(60,60^2)$ &  $77.53 _{\pm 17.41}$   & $\textbf{171.08}_{\pm 1.65}$  & $ \textbf{165.06}_{\pm 1.53}$ & $ 90.63 _{\pm 19.26}$  & $ 386.12_{\pm 2.05}$  & $ 365.86_{\pm 1.92}$   \\ \hline
 \multirow{3}*{$0.3$} & $\mathcal{N}(20,20^2)$ &  $85.55 _{\pm 18.26}$   & $\textbf{0.57  }_{\pm 0.01}$  & $ \textbf{5.25  }_{\pm 0.00}$ & $ 106.25_{\pm 38.59}$  & $ 0.58  _{\pm 0.01}$  & $ 5.26  _{\pm 0.00}$   \\ 
                      & $\mathcal{N}(40,40^2)$ &  $103.50_{\pm 27.60}$   & $\textbf{27.50 }_{\pm 0.18}$  & $ \textbf{41.44 }_{\pm 0.21}$ & $ 102.17_{\pm 34.68}$  & $ 69.02 _{\pm 0.35}$  & $ 89.88 _{\pm 0.42}$   \\
                      & $\mathcal{N}(60,60^2)$ &  $98.17 _{\pm 17.41}$   & $\textbf{190.22}_{\pm 1.72}$  & $ \textbf{197.55}_{\pm 1.76}$ & $ 68.74 _{\pm 14.46}$  & $ 552.93_{\pm 2.79}$  & $ 563.60_{\pm 2.82}$   \\
 \bottomrule
 \end{tabular}
\end{table}

\begin{table*}[h]
\caption{Comparisons of our \wxfreeRWB \ and the original free-support WB algorithm under different noise intensity on MNIST. We use $\zeta$ to denote the total mass of outliers, and the noise distribution (N.D.) is Gaussian distribution $\mathcal{N}(\cdot,\cdot)$.}
\label{Tab:free}
\begin{tabular}{llllllll}
\toprule
\multirow{2}*{$\zeta$} & \multirow{2}*{N.D.} & \multicolumn{3}{c}{Our \wxfreeRWB} & \multicolumn{2}{c}{WB} & \\
\cmidrule(lr){3-5}\cmidrule(lr){6-8}
& & runtime ($\downarrow$) & \textsf{WD} ($\downarrow$) & \textsf{cost} ($\downarrow$) & runtime ($\downarrow$) & \textsf{WD} ($\downarrow$) & \textsf{cost} ($\downarrow$) \\
\midrule
\multirow{3}*{$0.1$} & $\mathcal{N}(20,20^2)$ &  $406.26_{\pm 101.29}$ & $\textbf{0.56  }_{\pm 0.02 }$ & $\textbf{57.98 }_{\pm 0.02 }$ & $ 429.72_{\pm 125.86}$ & $ 0.57   _{\pm 0.02 }$ & $  57.99  _{\pm 0.02 }$   \\
& $\mathcal{N}(40,40^2)$ &  $374.55_{\pm 111.53}$ & $\textbf{2.86  }_{\pm 0.11 }$ & $\textbf{59.46 }_{\pm 0.11 }$ & $ 371.07_{\pm 109.45}$ & $ 91.27  _{\pm 31.51}$ & $  114.97 _{\pm 27.88}$   \\
& $\mathcal{N}(60,60^2)$ &  $299.06_{\pm 29.56 }$ & $\textbf{0.85  }_{\pm 0.04 }$ & $\textbf{58.27 }_{\pm 0.04 }$ & $ 313.12_{\pm 42.11 }$ & $ 473.66 _{\pm 7.86 }$ & $  444.91 _{\pm 6.86 }$   \\ \hline
\multirow{3}*{$0.2$} & $\mathcal{N}(20,20^2)$ &  $434.46_{\pm 71.85 }$ & $\textbf{2.50  }_{\pm 0.09 }$ & $\textbf{59.47 }_{\pm 0.09 }$ & $ 465.89_{\pm 105.96}$ & $ 2.56   _{\pm 0.09 }$ & $  59.52  _{\pm 0.09 }$   \\
& $\mathcal{N}(40,40^2)$ &  $462.63_{\pm 109.26}$ & $\textbf{15.78 }_{\pm 3.13 }$ & $\textbf{59.81 }_{\pm 1.01 }$ & $ 448.58_{\pm 67.21 }$ & $ 382.56 _{\pm 6.81 }$ & $  335.22 _{\pm 6.47 }$   \\
& $\mathcal{N}(60,60^2)$ &  $460.43_{\pm 66.94 }$ & $\textbf{3.71  }_{\pm 0.12 }$ & $\textbf{59.22 }_{\pm 0.26 }$ & $ 481.08_{\pm 56.36 }$ & $ 1201.98_{\pm 8.23 }$ & $  1054.11_{\pm 7.66 }$   \\ \hline
\multirow{3}*{$0.3$} & $\mathcal{N}(20,20^2)$ &  $477.99_{\pm 74.03 }$ & $\textbf{5.84  }_{\pm 0.11 }$ & $\textbf{61.41 }_{\pm 0.20 }$ & $ 509.11_{\pm 48.46 }$ & $ 5.99   _{\pm 0.11 }$ & $  61.51  _{\pm 0.21 }$   \\
& $\mathcal{N}(40,40^2)$ &  $306.84_{\pm 39.58 }$ & $\textbf{198.63}_{\pm 74.20}$ & $\textbf{174.58}_{\pm 62.32}$ & $ 312.85_{\pm 35.12 }$ & $ 661.14 _{\pm 7.50 }$ & $  540.72 _{\pm 7.90 }$   \\ 
& $\mathcal{N}(60,60^2)$ &  $444.65_{\pm 110.59}$ & $\textbf{10.86 }_{\pm 1.14 }$ & $\textbf{56.03 }_{\pm 0.76 }$ & $ 519.71_{\pm 76.13 }$ & $ 2093.29_{\pm 10.88}$ & $  1811.10_{\pm 9.60 }$   \\
\bottomrule\end{tabular}
\end{table*}

Then, we show the efficiency of our coreset technique in \Cref{Fig:sampling_1_40_100,Fig:sampling_1_40_200,Fig:sampling_1_40_300,Fig:sampling_1_60_100,Fig:sampling_1_60_200,Fig:sampling_1_60_300,Fig:sampling_1_80_100,Fig:sampling_1_80_200,Fig:sampling_1_80_300}. In this scenario, to obtain a noisy dataset $\wxcalQ'$, we first add $0.1$ mass of Gaussian noise from $\mathcal{N}(40,40)$ to each measure $\mu^l\in\wxcalQ$, and then shift the locations of several images randomly according to a distribution $\mathcal{N}(0,\cdot)$.  
Throughout our experiments, we ensure that the total sample size of the uniform sampling method equals to the coreset size of our method.
$\Gamma$ is the sample size for each layer in our method.
Our coreset method is much more time consuming. However, it performs well on the criteria \textsf{WD} and \textsf{cost}.
Moreover, our method is more stable.

\begin{figure}[htbp!]
    \centering
    \includegraphics[width=0.9\linewidth]{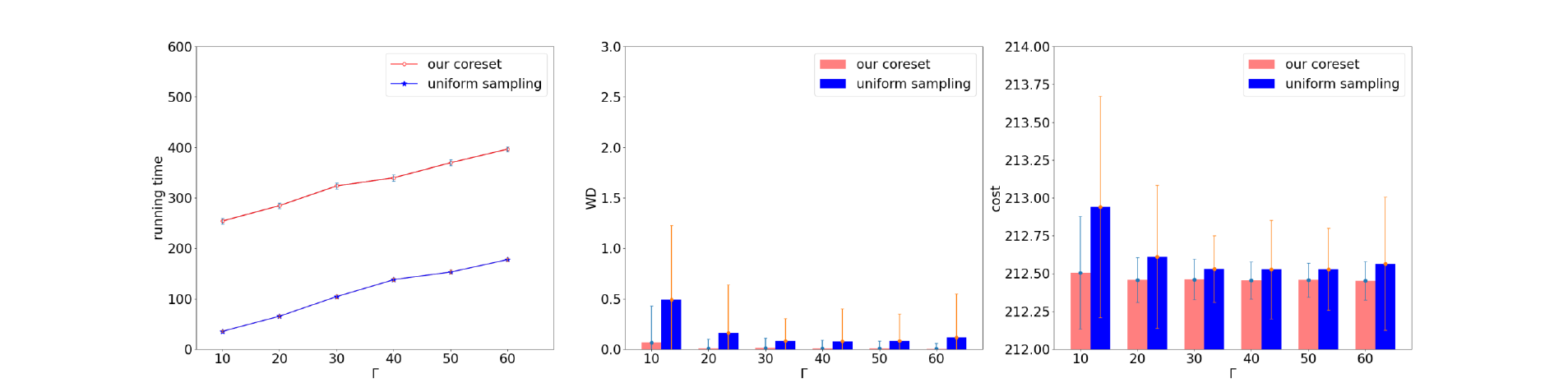}
    \caption{Comparisons of our coreset and uniform sampling on MNIST. Shift locations of 100 images according to distribution $\mathcal{N}(0,40)$.}
\label{Fig:sampling_1_40_100}
\end{figure}

\begin{figure}[htbp!]
    \centering
    \includegraphics[width=0.9\linewidth]{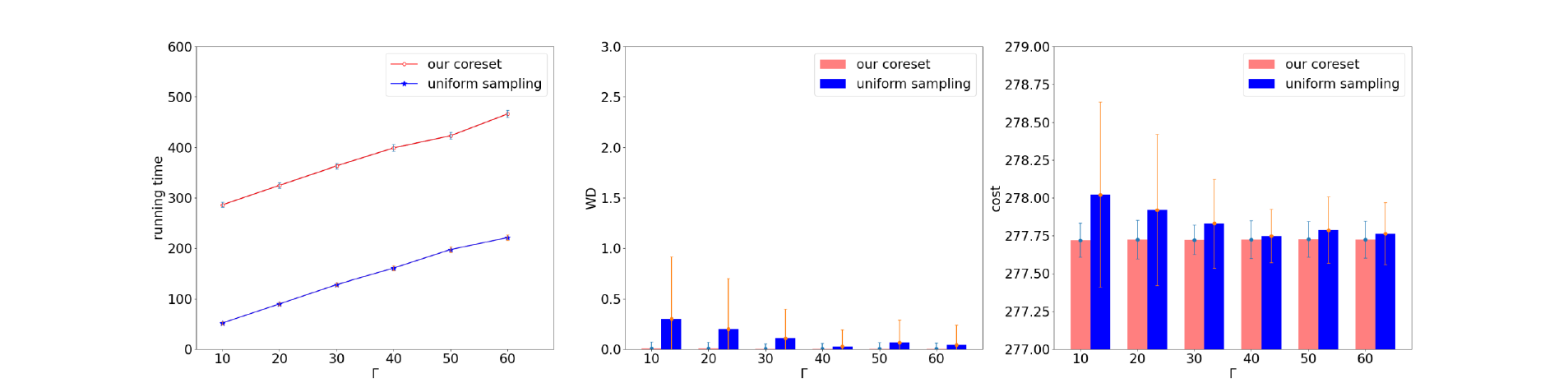}
    \caption{Comparisons of our coreset and uniform sampling on MNIST. Shift locations of 200 images according to distribution $\mathcal{N}(0,40)$.}
\label{Fig:sampling_1_40_200}
\end{figure}

\begin{figure}[htbp!]
    \centering
    \includegraphics[width=0.9\linewidth]{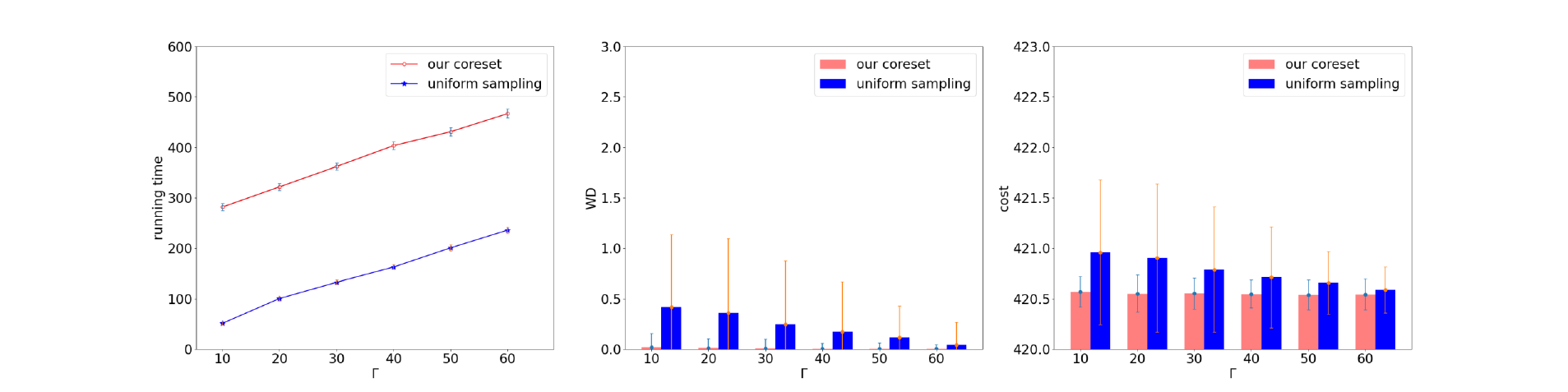}
    \caption{Comparisons of our coreset and uniform sampling on MNIST. Shift locations of 300 images according to distribution $\mathcal{N}(0,40)$.}
\label{Fig:sampling_1_40_300}
\end{figure}

\begin{figure}[htbp!]
    \centering
    \includegraphics[width=0.9\linewidth]{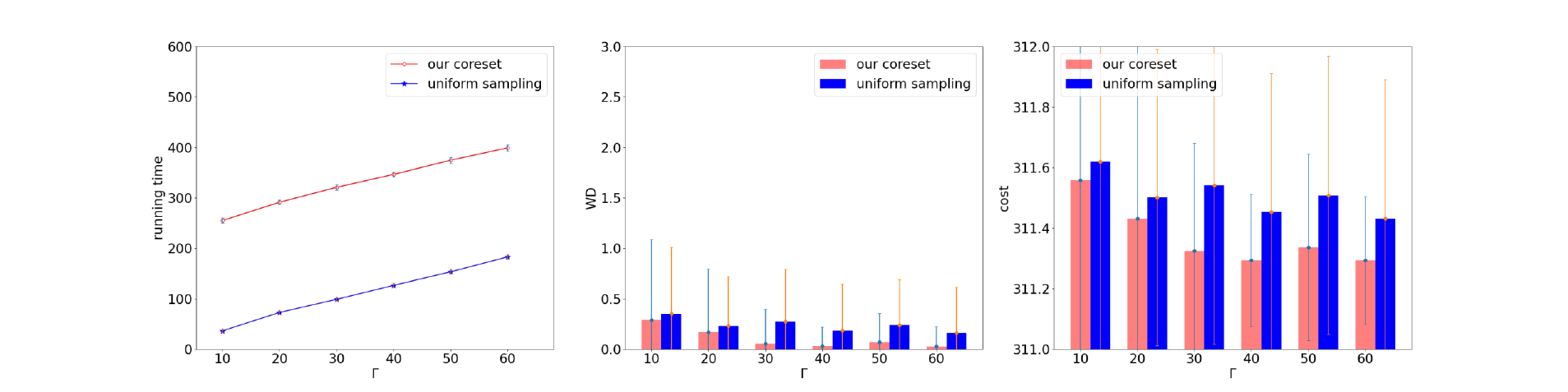}
    \caption{Comparisons of our coreset and uniform sampling on MNIST. Shift locations of 100 images according to distribution $\mathcal{N}(0,60)$.}
\label{Fig:sampling_1_60_100}
\end{figure}
   
\begin{figure}[htbp!]
    \centering
    \includegraphics[width=0.9\linewidth]{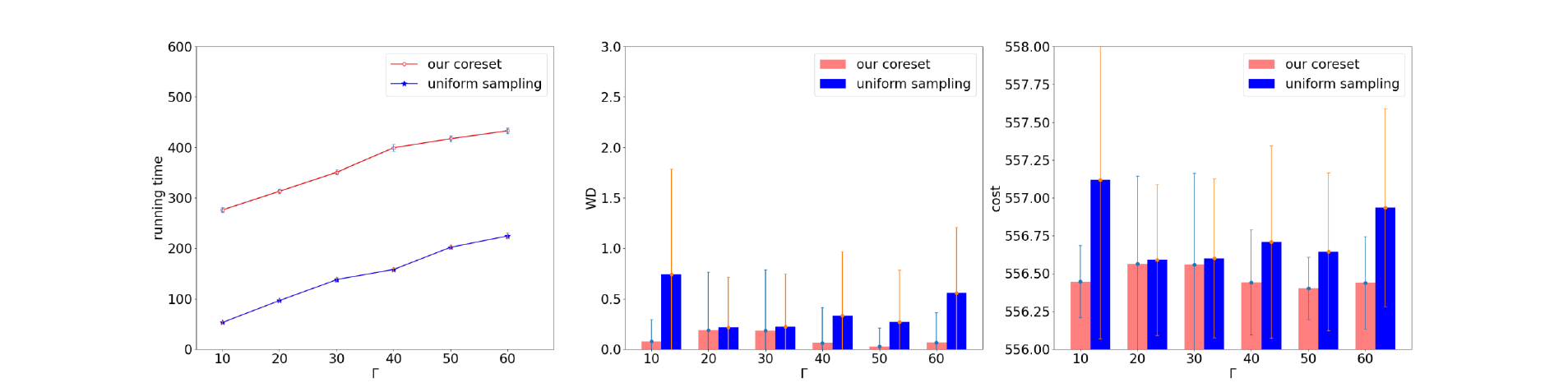}
    \caption{Comparisons of our coreset and uniform sampling on MNIST. Shift locations of 200 images according to distribution $\mathcal{N}(0,60)$.}
\label{Fig:sampling_1_60_200}
\end{figure}

\begin{figure}[htbp!]
    \centering
   \includegraphics[width=0.9\linewidth]{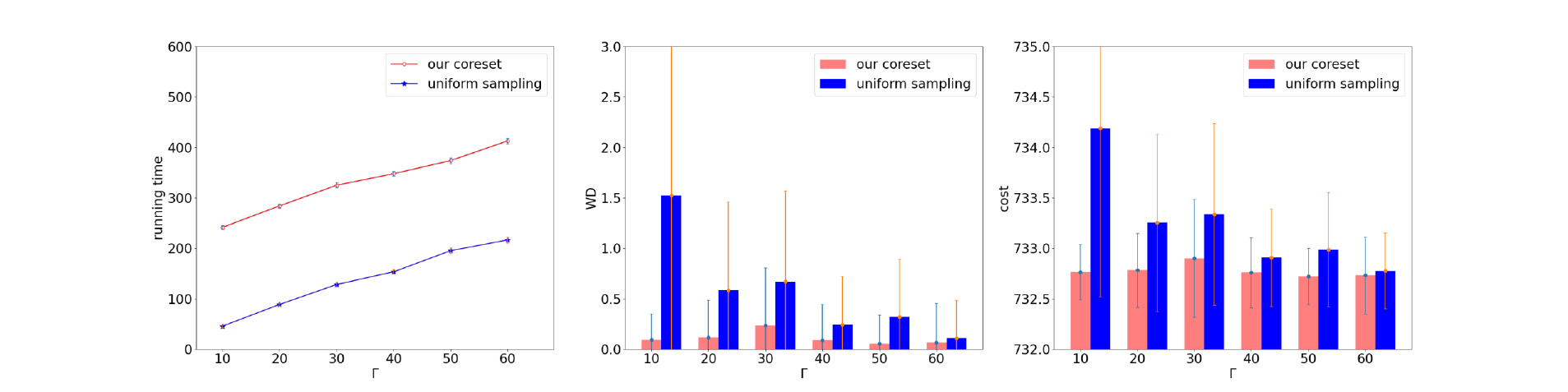}
    \caption{Comparisons of our coreset and uniform sampling on MNIST. Shift locations of 300 images according to distribution $\mathcal{N}(0,60)$.}
\label{Fig:sampling_1_60_300}
\end{figure}

\begin{figure}[htbp!]
    \centering
    \includegraphics[width=0.9\linewidth]{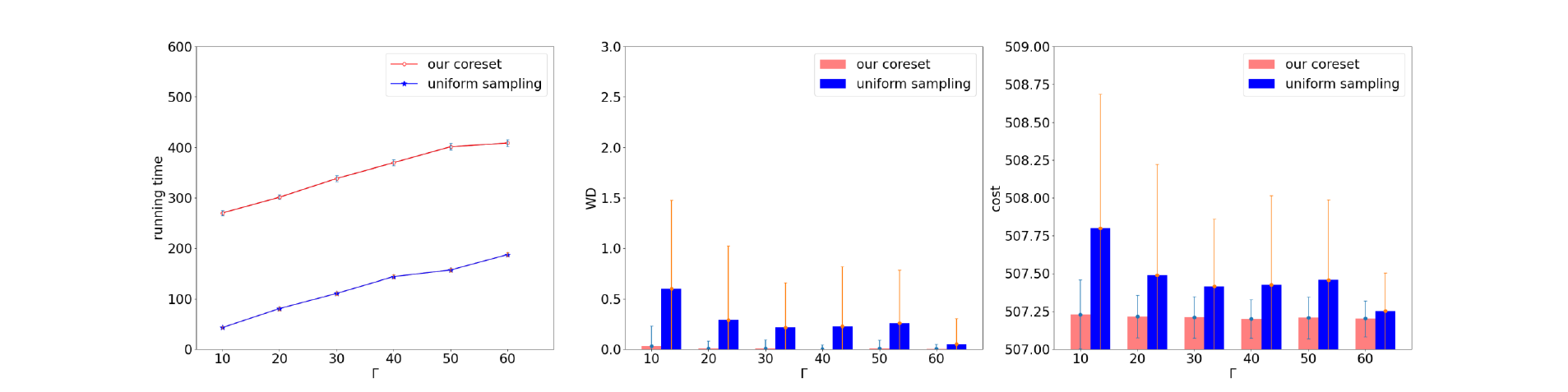}
    \caption{Comparisons of our coreset and uniform sampling on MNIST. Shift locations of 100 images according to distribution $\mathcal{N}(0,80)$.}
\label{Fig:sampling_1_80_100}
\end{figure}

\begin{figure}[htbp!]
    \centering
    \includegraphics[width=0.9\linewidth]{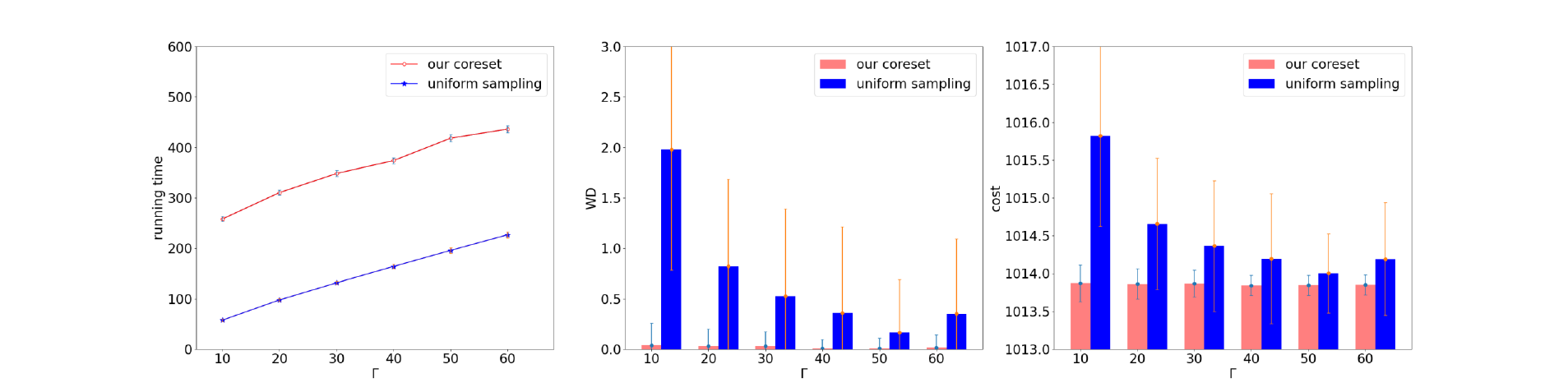}
    \caption{Comparisons of our coreset and uniform sampling on MNIST. Shift locations of 200 images according to distribution $\mathcal{N}(0,80)$.}
\label{Fig:sampling_1_80_200}
\end{figure}

\begin{figure}[htbp!]
    \centering
    \includegraphics[width=0.9\linewidth]{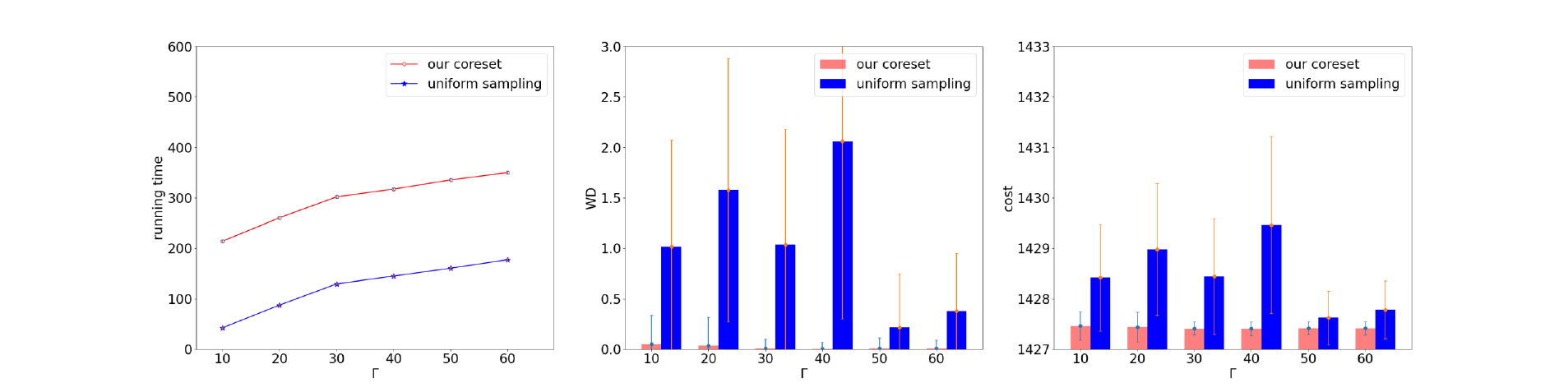}
    \caption{Comparisons of our coreset and uniform sampling on MNIST. Shift locations of 300 images according to distribution $\mathcal{N}(0,80)$.}
\label{Fig:sampling_1_80_300}
\end{figure}

\subsection{Experiments on HCP Dataset}

This section shows the experimental result on HCP dataset.
First, to show the efficiency of our \wxfixRWB/\wxfreeRWB, we add $\zeta$ mass of Gaussian noise to each measure $\mu^l\in\wxcalQ$ to obtain a noisy dataset $\wxcalQ'$. Then, we compute barycenter by the original fixed-support/free-support WB algorithm and our \wxfixRWB/\wxfreeRWB \ on the noisy dataset respectively. 
The results in \Cref{Tab:HCP_fixed}/\Cref{Tab:HCP_free} show that our \wxfixRWB/\wxfreeRWB \ can tackle outliers effectively under different noise intensity.

\begin{table}[htbp]
\caption{Comparisons of our \wxfixRWB \ and the original fixed-support WB algorithm under different noise intensity on HCP. We use $\zeta$ to denote the total mass of outliers, and the noise distribution (N.D.) is Gaussian distribution $\mathcal{N}(\cdot,\cdot)$.}
\label{Tab:HCP_fixed}
\begin{tabular}{llllllll}
 \toprule
 \multirow{2}*{$\zeta$} & \multirow{2}*{N.D.} & \multicolumn{3}{c}{Our \wxfixRWB} & \multicolumn{2}{c}{WB} & \\
 \cmidrule(lr){3-5}\cmidrule(lr){6-8}
 & & runtime ($\downarrow$) & \textsf{WD} ($\downarrow$) & \textsf{cost} ($\downarrow$) & runtime ($\downarrow$) & \textsf{WD} ($\downarrow$) & \textsf{cost} ($\downarrow$) \\
 \midrule 
 \multirow{3}*{$0.1$} & $\mathcal{N}(50,50^2)$ &  $101.50 _{\pm 35.27}$ & $ \textbf{4.98  } _{\pm 0.12}$ & $ 97.39  _{\pm 0.05} $ & $ 92.71 _{\pm 25.00}$ & $ 5.20    _{\pm 0.11 }$ & $ \textbf{97.34 }  _{0.05 }$ \\                    
                      & $\mathcal{N}(100,100^2)$ &  $84.28  _{\pm 29.66}$ & $ \textbf{89.45 } _{\pm 1.39}$ & $ \textbf{179.87} _{\pm 1.37} $ & $ 105.27_{\pm 25.10}$ & $ 321.63  _{\pm 1.72 }$ & $ 404.51  _{1.67 }$ \\                    
                      & $\mathcal{N}(150,150^2)$ &  $83.36  _{\pm 6.68 }$ & $ \textbf{82.40 } _{\pm 3.90}$ & $ \textbf{176.28} _{\pm 3.94} $ & $ 89.49 _{\pm 24.22}$ & $ 2039.19 _{\pm 7.51 }$ & $ 2134.79 _{7.46 }$ \\   \hline                 
 \multirow{3}*{$0.2$} & $\mathcal{N}(50,50^2)$ &  $61.77  _{\pm 16.90}$ & $ \textbf{10.15 } _{\pm 0.23}$ & $ 96.36  _{\pm 0.05} $ & $ 65.90 _{\pm 16.08}$ & $ 10.78   _{\pm 0.24 }$ & $ \textbf{96.24}   _{0.05 }$ \\                    
                      & $\mathcal{N}(100,100^2)$ &  $77.48  _{\pm 17.16}$ & $ \textbf{169.68} _{\pm 1.28}$ & $ \textbf{256.04} _{\pm 1.29} $ & $ 72.97 _{\pm 18.30}$ & $ 718.88  _{\pm 3.54 }$ & $ 788.35  _{3.40 }$ \\                    
                      & $\mathcal{N}(150,150^2)$ &  $52.19  _{\pm 3.50 }$ & $ \textbf{231.62} _{\pm 5.89}$ & $ \textbf{327.27} _{\pm 6.05} $ & $ 49.86 _{\pm 1.99 }$ & $ 4505.75 _{\pm 18.61}$ & $ 4584.84 _{18.32}$ \\   \hline                  
 \multirow{3}*{$0.3$} & $\mathcal{N}(50,50^2)$ &  $95.91  _{\pm 33.15}$ & $ \textbf{15.40 } _{\pm 0.09}$ & $ 95.66  _{\pm 0.04} $ & $ 107.73_{\pm 22.91}$ & $ 16.60   _{\pm 0.10 }$ & $ \textbf{95.55}   _{0.05 }$ \\                    
                      & $\mathcal{N}(100,100^2)$ &  $93.63  _{\pm 27.43}$ & $ \textbf{184.54} _{\pm 2.30}$ & $ \textbf{268.10} _{\pm 2.34} $ & $ 99.77 _{\pm 25.78}$ & $ 1180.03 _{\pm 6.22 }$ & $ 1227.38 _{6.04 }$ \\                    
                      & $\mathcal{N}(150,150^2)$ &  $49.84  _{\pm 3.15 }$ & $ \textbf{422.77} _{\pm 6.48}$ & $ \textbf{521.64} _{\pm 6.48} $ & $ 50.87 _{\pm 5.31 }$ & $ 7182.62 _{\pm 16.07}$ & $ 7211.00 _{15.77}$ \\                     
 \bottomrule
 \end{tabular}
\end{table}

\begin{table}[htbp]
\caption{Comparisons of our \wxfreeRWB \ and the original free-support WB algorithm under different noise intensity on HCP. We use $\zeta$ to denote the total mass of outliers, and the noise distribution (N.D.) is Gaussian distribution $\mathcal{N}(\cdot,\cdot)$.}
\label{Tab:HCP_free}
\begin{tabular}{llllllll}
 \toprule
 \multirow{2}*{$\zeta$} & \multirow{2}*{N.D.} & \multicolumn{3}{c}{Our \wxfreeRWB} & \multicolumn{2}{c}{WB} & \\
 \cmidrule(lr){3-5}\cmidrule(lr){6-8}
 & & runtime ($\downarrow$) & \textsf{WD} ($\downarrow$) & \textsf{cost} ($\downarrow$) & runtime ($\downarrow$) & \textsf{WD} ($\downarrow$) & \textsf{cost} ($\downarrow$) \\
 \midrule
 \multirow{3}*{$0.1$} & $\mathcal{N}(50,50^2)$ & $364.40 _{\pm 55.35}$ & $ \textbf{2.79 } _{\pm 0.10}$ & $ 924.84 _{\pm 0.13}$ & $ 391.94 _{\pm 46.67}$ & $ 3.19     _{\pm 0.10  }$ & $ \textbf{924.63}   _{\pm 0.17 }$  \\
                      & $\mathcal{N}(100,100^2)$ & $362.76 _{\pm 41.85}$ & $ \textbf{2.53 } _{\pm 0.08}$ & $ \textbf{925.81} _{\pm 0.16}$ & $ 393.62 _{\pm 51.49}$ & $ 1428.06  _{\pm 53.70 }$ & $ 1758.54  _{\pm 36.74}$  \\ 
                      & $\mathcal{N}(150,150^2)$ & $377.66 _{\pm 68.12}$ & $ \textbf{0.09 } _{\pm 0.01}$ & $ \textbf{927.67} _{\pm 0.03}$ & $ 390.97 _{\pm 72.66}$ & $ 4912.27  _{\pm 52.62 }$ & $ 4581.39  _{\pm 48.59}$  \\  \hline
 \multirow{3}*{$0.2$} & $\mathcal{N}(50,50^2)$ & $385.24 _{\pm 74.60}$ & $ \textbf{11.78} _{\pm 0.56}$ & $ 910.16 _{\pm 0.74}$ & $ 402.37 _{\pm 74.39}$ & $ 14.10    _{\pm 0.62  }$ & $ \textbf{906.11}   _{\pm 0.72 }$  \\
                      & $\mathcal{N}(100,100^2)$ & $392.26 _{\pm 73.57}$ & $ \textbf{9.96 } _{\pm 0.32}$ & $ \textbf{918.17} _{\pm 1.02}$ & $ 385.29 _{\pm 50.78}$ & $ 3962.80  _{\pm 179.01}$ & $ 3409.19  _{\pm 147.24}$ \\ 
                      & $\mathcal{N}(150,150^2)$ & $347.45 _{\pm 27.78}$ & $ \textbf{0.30 } _{\pm 0.02}$ & $ \textbf{927.10} _{\pm 0.06}$ & $ 400.30 _{\pm 25.87}$ & $ 12873.93 _{\pm 88.97 }$ & $ 11040.97 _{\pm 78.38}$  \\  \hline
 \multirow{3}*{$0.3$} & $\mathcal{N}(50,50^2)$ & $342.88 _{\pm 24.33}$ & $ \textbf{30.53} _{\pm 0.50}$ & $ 867.25 _{\pm 1.57}$ & $ 377.99 _{\pm 45.68}$ & $ 38.19    _{\pm 0.77  }$ & $ \textbf{847.05}   _{\pm 2.33 }$  \\
                      & $\mathcal{N}(100,100^2)$ & $347.99 _{\pm 43.53}$ & $ \textbf{30.51} _{\pm 7.15}$ & $ \textbf{889.46} _{\pm 4.19}$ & $ 400.05 _{\pm 40.38}$ & $ 7414.53  _{\pm 51.70 }$ & $ 5896.24  _{\pm 53.92}$  \\ 
                      & $\mathcal{N}(150,150^2)$ & $43.53  _{\pm 56.00}$ & $ \textbf{7.15 } _{\pm 0.03}$ & $ \textbf{4.19  } _{\pm 0.09}$ & $ 40.38  _{\pm 26.55}$ & $ 51.70    _{\pm 87.81 }$ & $ 53.92    _{\pm 76.20}$  \\ 
 \bottomrule
 \end{tabular}
\end{table}

Then, we show the efficiency of our coreset technique in \Cref{Fig:HCP_sampling_1_50_100,Fig:HCP_sampling_1_50_200,Fig:HCP_sampling_1_50_300,Fig:HCP_sampling_1_100_100,Fig:HCP_sampling_1_100_200,Fig:HCP_sampling_1_100_300,Fig:HCP_sampling_1_150_100,Fig:HCP_sampling_1_150_200,Fig:HCP_sampling_1_150_300}. In this scenario, to obtain a noisy dataset $\wxcalQ'$, we first add $0.1$ mass of Gaussian noise from $\mathcal{N}(40,40)$ to each measure $\mu^l\in\wxcalQ$, and then shift the locations of several images randomly according to a distribution $\mathcal{N}(0,\cdot)$.  
Throughout our experiments, we ensure that the total sample size of the uniform sampling method equals to the coreset size of our method.
$\Gamma$ is the sample size for each layer in our method.
Our coreset method is much more time consuming. However, it performs well on the criteria \textsf{WD} and \textsf{cost}.
Moreover, our method is more stable.
(The results of our coreset method perform well with high probability, so it is reasonable that our method performs worse than uniform sampling with small probability.)

\begin{figure}[htbp!]
    \centering
    \includegraphics[width=0.9\linewidth]{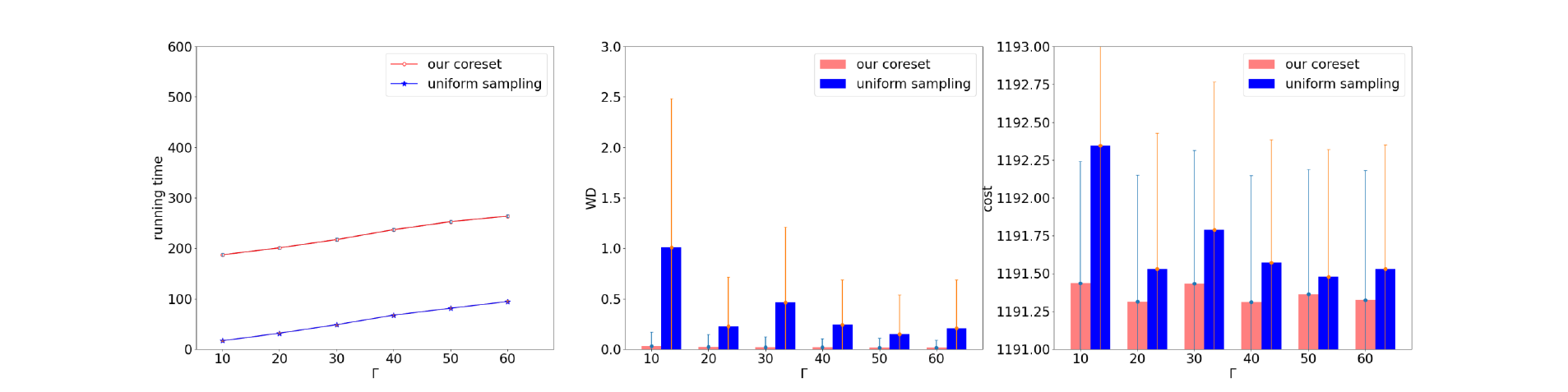}
    \caption{Comparisons of our coreset and uniform sampling on HCP. Shift locations of 100 brains according to distribution $\mathcal{N}(0,50)$.}
\label{Fig:HCP_sampling_1_50_100}
\end{figure}

\begin{figure}[htbp!]
    \centering
    \includegraphics[width=0.9\linewidth]{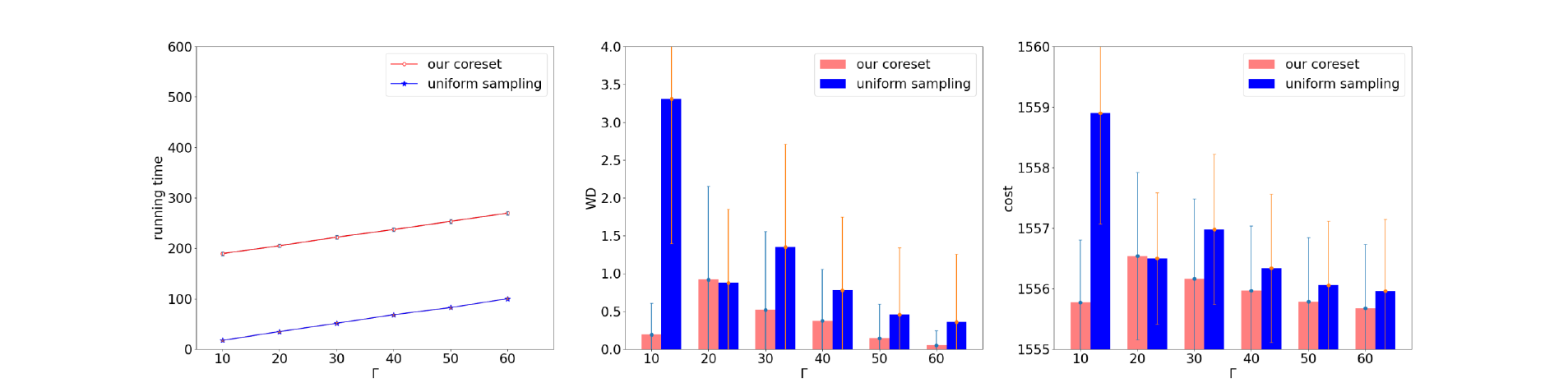}
    \caption{Comparisons of our coreset and uniform sampling on HCP. Shift locations  of 200 brains according to distribution $\mathcal{N}(0,50)$.}
\label{Fig:HCP_sampling_1_50_200}
\end{figure}

\begin{figure}[htbp!]
    \centering
    \includegraphics[width=0.9\linewidth]{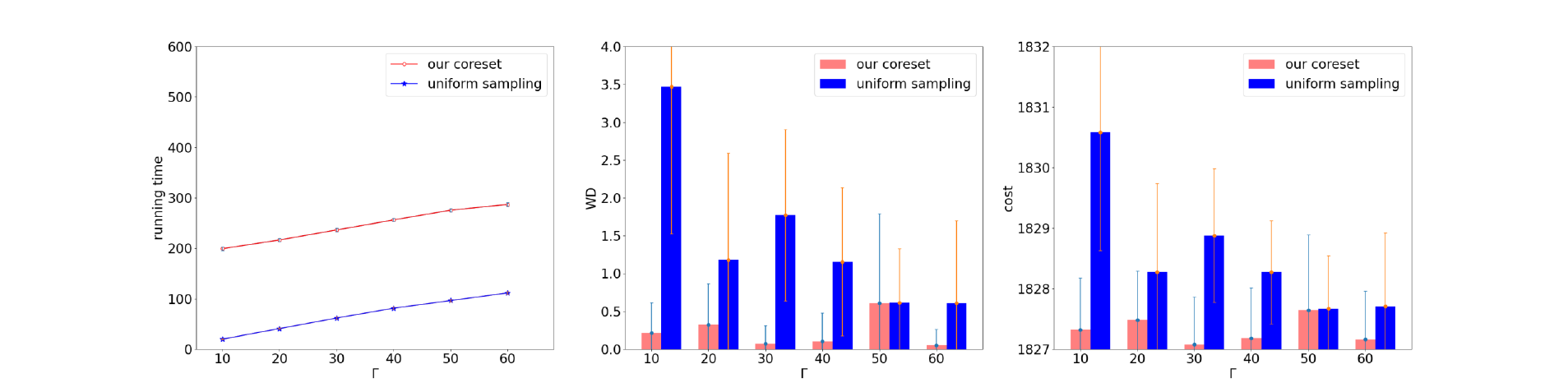}
    \caption{Comparisons of our coreset and uniform sampling on HCP. Shift locations  of 300 brains according to distribution $\mathcal{N}(0,50)$.}
\label{Fig:HCP_sampling_1_50_300}
\end{figure}

\begin{figure}[htbp!]
    \centering
    \includegraphics[width=0.9\linewidth]{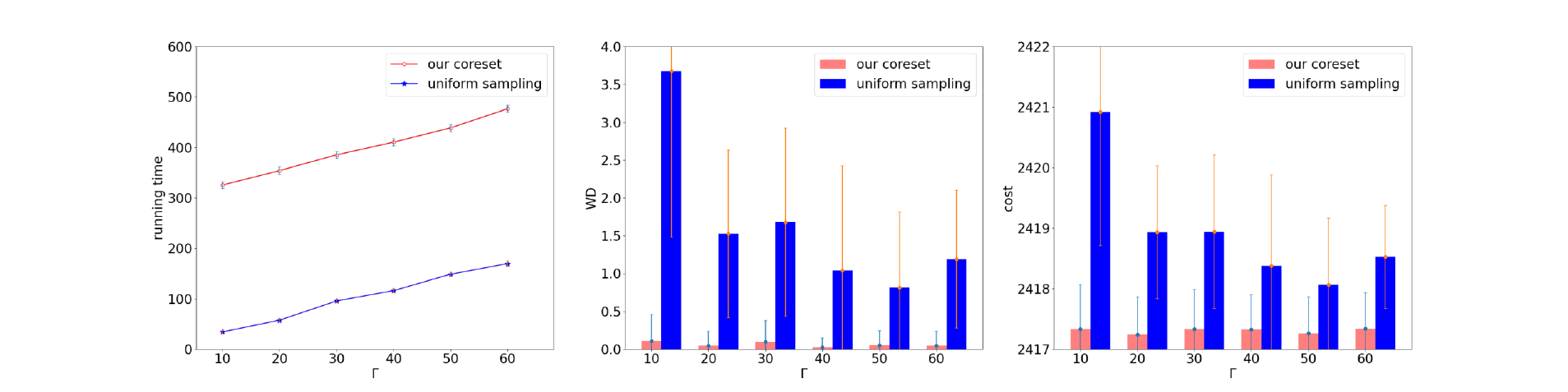}
    \caption{Comparisons of our coreset and uniform sampling on HCP. Shift locations of 100 brains according to distribution $\mathcal{N}(0,100)$.}
\label{Fig:HCP_sampling_1_100_100}
\end{figure}

\begin{figure}[htbp!]
    \centering
    \includegraphics[width=0.9\linewidth]{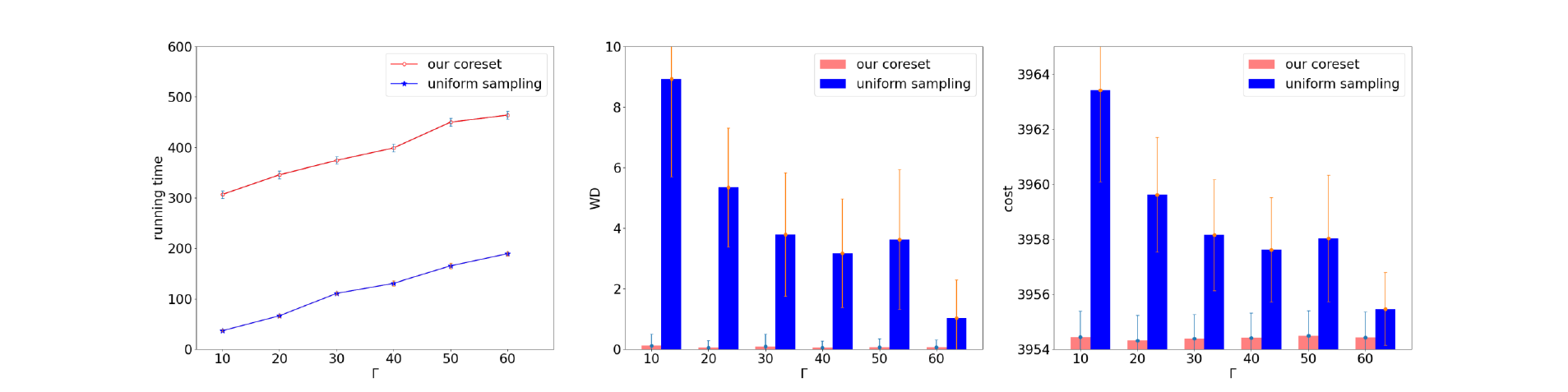}
    \caption{Comparisons of our coreset and uniform sampling on HCP. Shift locations  of 200 brains according to distribution $\mathcal{N}(0,100)$.}
\label{Fig:HCP_sampling_1_100_200}
\end{figure}

\begin{figure}[htbp!]
    \centering
    \includegraphics[width=0.9\linewidth]{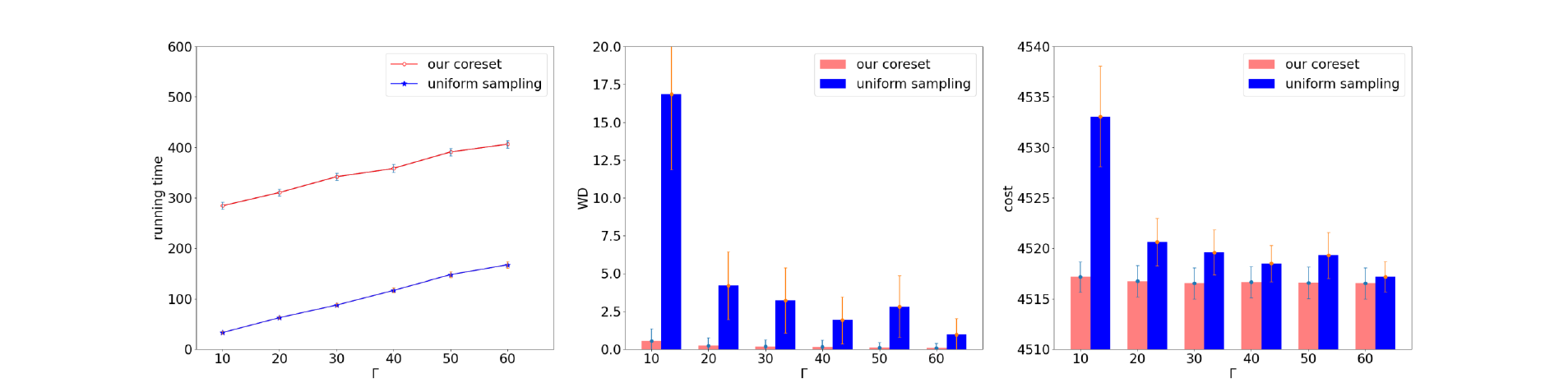}
    \caption{Comparisons of our coreset and uniform sampling on HCP. Shift locations  of 300 brains according to distribution $\mathcal{N}(0,100)$.}
\label{Fig:HCP_sampling_1_100_300}
\end{figure}

\begin{figure}[htbp!]
    \centering
    \includegraphics[width=0.9\linewidth]{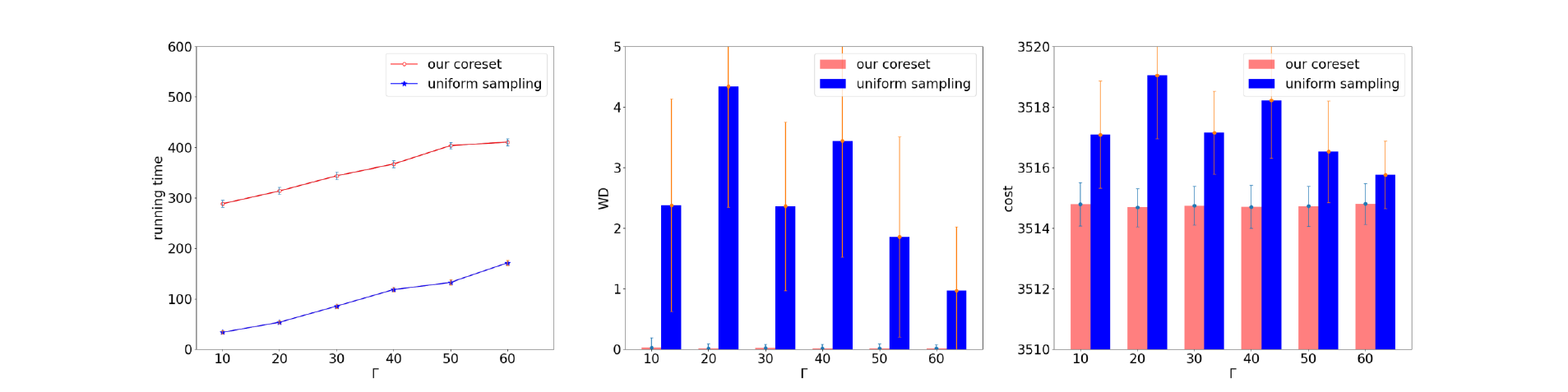}
    \caption{Comparisons of our coreset and uniform sampling on HCP. Shift locations of 100 brains according to distribution $\mathcal{N}(0,150)$.}
\label{Fig:HCP_sampling_1_150_100}
\end{figure}

\begin{figure}[htbp!]
    \centering
    \includegraphics[width=0.9\linewidth]{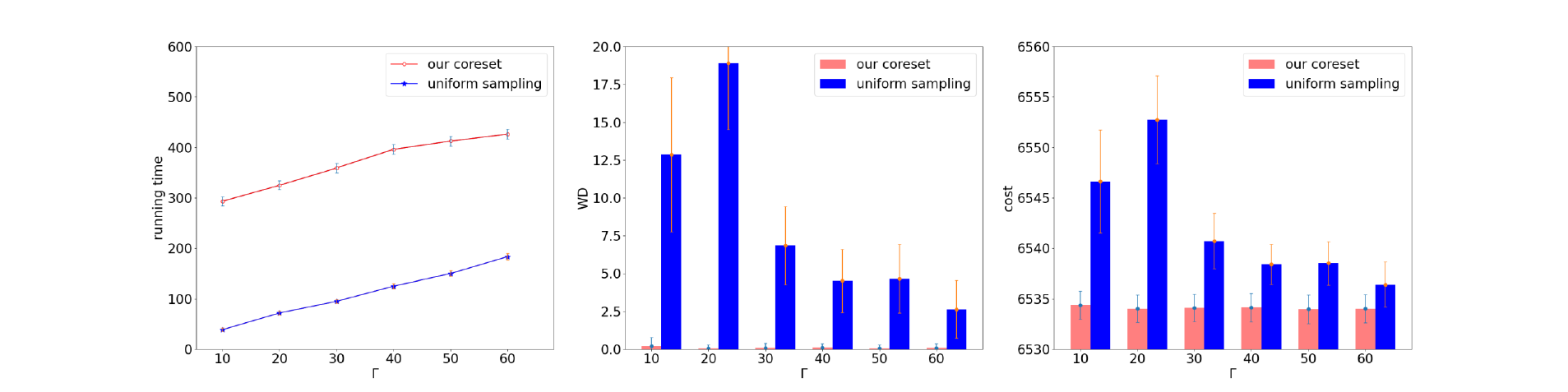}
    \caption{Comparisons of our coreset and uniform sampling on HCP. Shift locations of 200 brains according to distribution $\mathcal{N}(0,150)$.}
\label{Fig:HCP_sampling_1_150_200}
\end{figure}

\begin{figure}[htbp!]
    \centering
    \includegraphics[width=0.9\linewidth]{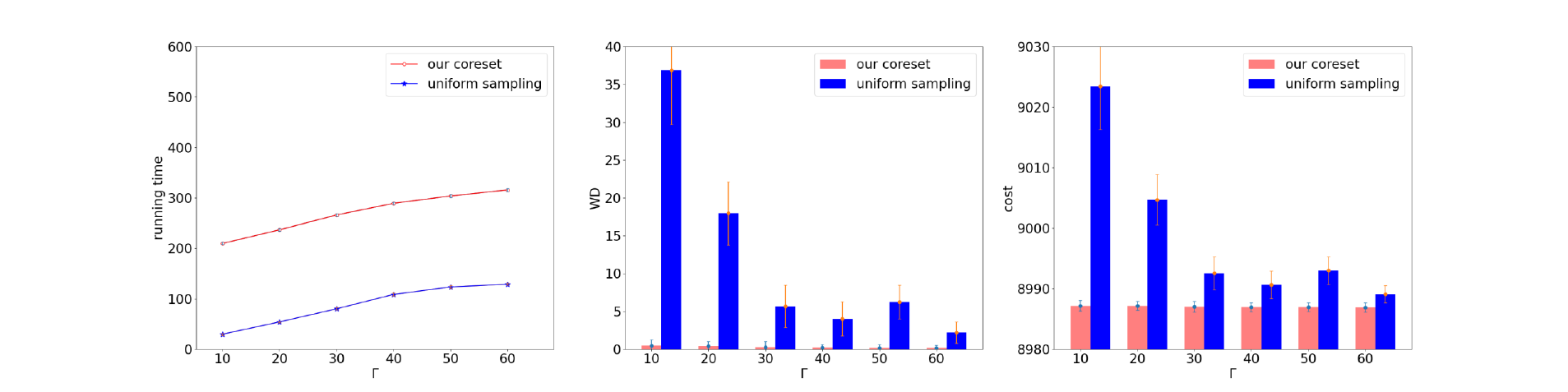}
    \caption{Comparisons of our coreset and uniform sampling on HCP. Shift locations of 300 brains according to distribution $\mathcal{N}(0,150)$.}
\label{Fig:HCP_sampling_1_150_300}
\end{figure}

\subsection{Experiments on ModelNet40 Dataset}

This section shows the experimental result on ModelNet40 dataset.
First, to show the efficiency of our \wxfixRWB/\wxfreeRWB, we add $\zeta$ mass of Gaussian noise to each measure $\mu^l\in\wxcalQ$ to obtain a noisy dataset $\wxcalQ'$. Then, we compute barycenter by the original fixed-support/free-support WB algorithm and our \wxfixRWB/\wxfreeRWB \ on the noisy dataset respectively. 
The results in \Cref{Tab:MN_fixed}/\Cref{Tab:MN_free} show that our \wxfixRWB/\wxfreeRWB \ can tackle outliers effectively under different noise intensity.
In our experiments, we can only obtain a local approximate solution, the solution is somewhat random. Thus, it is reasonable that our method obtain worse result when the noise intensity is small.

\begin{table}[htbp]
\caption{Comparisons of our \wxfixRWB \ and the original fixed-support WB algorithm under different noise intensity on ModelNet40. We use $\zeta$ to denote the total mass of outliers, and the noise distribution (N.D.) is Gaussian distribution $\mathcal{N}(\cdot,\cdot)$.}
\label{Tab:MN_fixed}
\begin{tabular}{llllllll}
 \toprule
 \multirow{2}*{$\zeta$} & \multirow{2}*{N.D.} & \multicolumn{3}{c}{Our \wxfixRWB} & \multicolumn{2}{c}{WB} & \\
 \cmidrule(lr){3-5}\cmidrule(lr){6-8}
 & & runtime ($\downarrow$) & \textsf{WD} ($\downarrow$) & \textsf{cost} ($\downarrow$) & runtime ($\downarrow$) & \textsf{WD} ($\downarrow$) & \textsf{cost} ($\downarrow$) \\
 \midrule
 \multirow{3}*{$0.1$} & $\mathcal{N}(50,50^2)$ &  $32.63 _{\pm 11.52}$ & $ \textbf{73.57  } _{\pm 0.58  }$ & $ \textbf{187.15 } _{\pm 0.48  }$ & $ 34.98 _{\pm 9.08 }$ & $ 79.87    _{\pm 0.73 }$ &    $192.85   _{\pm 0.56   }$ \\                  
                      & $\mathcal{N}(100,100^2)$ &  $10.72 _{\pm 43.08}$ & $ \textbf{7.87   } _{\pm 572.59}$ & $ \textbf{8.08   } _{\pm 697.28}$ & $ 36.17 _{\pm 37.89}$ & $ 1172.51  _{\pm 3438.73}$ &  $1278.68  _{\pm 3560.87}$ \\                    
                      & $\mathcal{N}(150,150^2)$ &  $31.16 _{\pm 5.56 }$ & $ \textbf{172.03 } _{\pm 1.14  }$ & $ \textbf{259.79 } _{\pm 1.09  }$ & $ 32.70 _{\pm 7.09 }$ & $ 194.93   _{\pm 1.53 }$ &    $280.03   _{\pm 1.46   }$ \\   \hline               
 \multirow{3}*{$0.2$} & $\mathcal{N}(50,50^2)$ &  $32.37 _{\pm 6.97 }$ & $ \textbf{1141.13} _{\pm 11.48 }$ & $ \textbf{1242.84} _{\pm 11.41 }$ & $ 33.31 _{\pm 6.10 }$ & $ 2677.97  _{\pm 9.06 }$ &    $2760.90  _{\pm 8.88   }$ \\                  
                      & $\mathcal{N}(100,100^2)$ &  $38.34 _{\pm 9.18 }$ & $ \textbf{1456.09} _{\pm 24.07 }$ & $ \textbf{1583.98} _{\pm 24.01 }$ & $ 37.31 _{\pm 11.37}$ & $ 7526.14  _{\pm 50.15}$ &    $7638.26  _{\pm 50.09  }$ \\                   
                      & $\mathcal{N}(150,150^2)$ &  $32.30 _{\pm 8.32 }$ & $ \textbf{276.55 } _{\pm 3.00  }$ & $ \textbf{344.85 } _{\pm 2.82  }$ & $ 32.48 _{\pm 8.99 }$ & $ 324.69   _{\pm 3.40 }$ &    $388.14   _{\pm 3.08   }$ \\   \hline               
 \multirow{3}*{$0.3$} & $\mathcal{N}(50,50^2)$ &  $40.42 _{\pm 8.82 }$ & $ \textbf{1368.73} _{\pm 16.04 }$ & $ \textbf{1465.01} _{\pm 16.19 }$ & $ 43.17 _{\pm 10.26}$ & $ 4314.01  _{\pm 25.69}$ &    $4335.73  _{\pm 24.59  }$ \\                   
                      & $\mathcal{N}(100,100^2)$ &  $23.20 _{\pm 1.56 }$ & $ \textbf{2283.81} _{\pm 41.22 }$ & $ \textbf{2402.69} _{\pm 41.07 }$ & $ 21.27 _{\pm 1.38 }$ & $ 12023.66 _{\pm 31.14}$ &    $12025.97 _{\pm 30.12  }$ \\                   
                      & $\mathcal{N}(150,150^2)$ &  $38.18 _{\pm 11.32}$ & $ \textbf{2279.04} _{\pm 27.75 }$ & $ \textbf{2398.06} _{\pm 27.65 }$ & $ 35.47 _{\pm 6.61 }$ & $ 12032.18 _{\pm 40.15}$ &    $12034.24 _{\pm 38.74  }$ \\                   
 \bottomrule
 \end{tabular}
\end{table}

\begin{table}[htbp]
\caption{Comparisons of our \wxfreeRWB \ and the original free-support WB algorithm under different noise intensity on ModelNet40. We use $\zeta$ to denote the total mass of outliers, and the noise distribution (N.D.) is Gaussian distribution $\mathcal{N}(\cdot,\cdot)$.}
\label{Tab:MN_free}
\begin{tabular}{llllllll}
 \toprule
 \multirow{2}*{$\zeta$} & \multirow{2}*{N.D.} & \multicolumn{3}{c}{Our \wxfreeRWB} & \multicolumn{2}{c}{WB} & \\
 \cmidrule(lr){3-5}\cmidrule(lr){6-8}
 & & runtime ($\downarrow$) & \textsf{WD} ($\downarrow$) & \textsf{cost} ($\downarrow$) & runtime ($\downarrow$) & \textsf{WD} ($\downarrow$) & \textsf{cost} ($\downarrow$) \\
 \midrule
 \multirow{3}*{$0.1$} & $\mathcal{N}(50,50^2)$ & $165.92 _{\pm 25.89}$ & $\textbf{161.05}  _{\pm 85.74 }$ & $ \textbf{700.58}  _{\pm 37.09 }$ & $ 160.59 _{\pm 14.37}$ & $ 305.74   _{\pm 28.49 }$ & $ 759.37   _{\pm 22.66 }$ \\
                      & $\mathcal{N}(100,100^2)$ & $147.06 _{\pm 18.13}$ & $\textbf{5.17  }  _{\pm 0.31  }$ & $ \textbf{694.28}  _{\pm 1.88  }$ & $ 154.87 _{\pm 14.97}$ & $ 2727.47  _{\pm 63.60 }$ & $ 2524.60  _{\pm 54.60 }$ \\
                      & $\mathcal{N}(150,150^2)$ & $162.39 _{\pm 12.54}$ & $\textbf{0.34  }  _{\pm 0.05  }$ & $ \textbf{697.43}  _{\pm 0.13  }$ & $ 159.90 _{\pm 14.85}$ & $ 6485.18  _{\pm 98.90 }$ & $ 5754.88  _{\pm 81.64 }$ \\ \hline
 \multirow{3}*{$0.2$} & $\mathcal{N}(50,50^2)$ & $147.41 _{\pm 25.81}$ & $1047.55 _{\pm 42.76 }$ & $ 1143.70 _{\pm 32.79 }$ & $ 148.08 _{\pm 23.28}$ & $ \textbf{1036.94}  _{\pm 27.18 }$ & $ \textbf{1132.61}  _{\pm 22.44 }$ \\
                      & $\mathcal{N}(100,100^2)$ & $133.21 _{\pm 10.56}$ & $\textbf{472.29}  _{\pm 507.73}$ & $ \textbf{918.99}  _{\pm 306.73}$ & $ 147.36 _{\pm 10.85}$ & $ 6725.00  _{\pm 75.05 }$ & $ 5554.06  _{\pm 60.25 }$ \\
                      & $\mathcal{N}(150,150^2)$ & $156.13 _{\pm 21.13}$ & $\textbf{2.06  }  _{\pm 0.18  }$ & $ \textbf{697.74}  _{\pm 0.34  }$ & $ 178.76 _{\pm 27.98}$ & $ 16143.00 _{\pm 226.85}$ & $ 13870.42 _{\pm 199.02}$ \\ \hline
 \multirow{3}*{$0.3$} & $\mathcal{N}(50,50^2)$ & $162.40 _{\pm 13.90}$ & $1899.46 _{\pm 46.76 }$ & $ 1651.95 _{\pm 35.70 }$ & $ 161.80 _{\pm 21.29}$ & $ \textbf{1836.57}  _{\pm 35.69 }$ & $ \textbf{1604.22}  _{\pm 28.64 }$ \\
                      & $\mathcal{N}(100,100^2)$ & $160.68 _{\pm 13.73}$ & $\textbf{2498.13} _{\pm 751.33}$ & $ \textbf{2245.65} _{\pm 543.33}$ & $ 203.94 _{\pm 20.47}$ & $ 11554.34 _{\pm 149.35}$ & $ 9428.57  _{\pm 131.31}$ \\ 
                      & $\mathcal{N}(150,150^2)$ & $157.10 _{\pm 24.86}$ & $\textbf{5.05   } _{\pm 0.30  }$ & $ \textbf{698.02 } _{\pm 0.61  }$ & $ 211.04 _{\pm 30.49}$ & $ 26774.86 _{\pm 243.22}$ & $ 23069.52 _{\pm 222.73}$ \\ 
 \bottomrule
 \end{tabular}
\end{table}

Then, we show the efficiency of our coreset technique in \Cref{Fig:MN_sampling_1_50_30,Fig:MN_sampling_1_50_60,Fig:MN_sampling_1_50_90,Fig:MN_sampling_1_100_30,Fig:MN_sampling_1_100_60,Fig:MN_sampling_1_100_90,Fig:MN_sampling_150_30,Fig:MN_sampling_1_150_60,Fig:MN_sampling_1_150_90}. In this scenario, to obtain a noisy dataset $\wxcalQ'$, we first add $0.1$ mass of Gaussian noise from $\mathcal{N}(40,40)$ to each measure $\mu^l\in\wxcalQ$, and then shift the locations of several images randomly according to a distribution $\mathcal{N}(0,\cdot)$.  
Throughout our experiments, we ensure that the total sample size of the uniform sampling method equals to the coreset size of our method.
$\Gamma$ is the sample size for each layer in our method.
Our coreset method is much more time consuming. However, it performs well on the criteria \textsf{WD} and \textsf{cost}.
Moreover, our method is more stable.
(The results of our coreset method perform well with high probability, so it is reasonable that our method performs worse than uniform sampling with small probability.)

\begin{figure}[htbp!]
    \centering
    \includegraphics[width=0.9\linewidth]{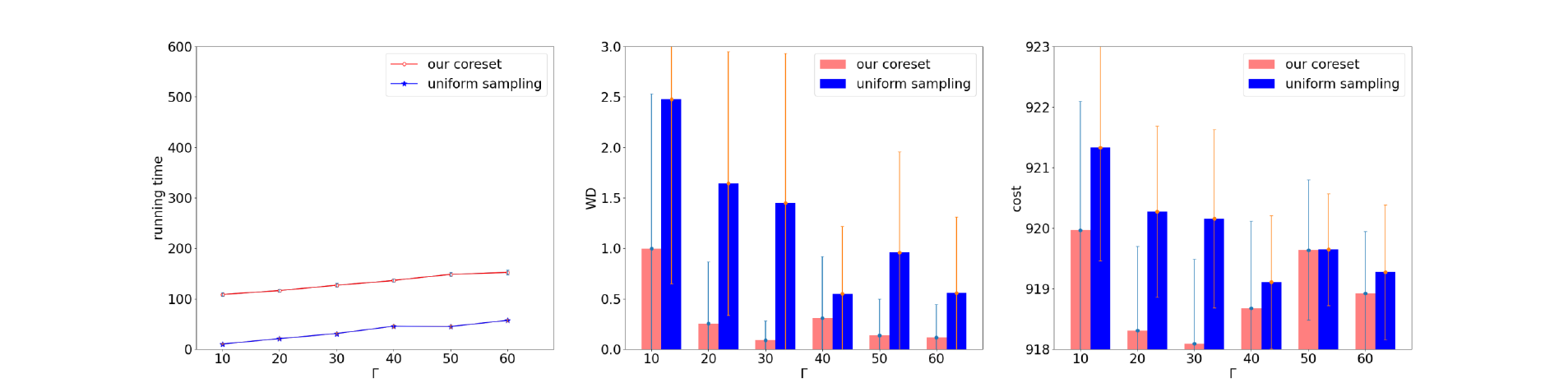}
    \caption{Comparisons of our coreset and uniform sampling on ModelNet40. Shift locations of 30 CAD models according to distribution $\mathcal{N}(0,50)$.}
\label{Fig:MN_sampling_1_50_30}
\end{figure}

\begin{figure}[htbp!]
    \centering
    \includegraphics[width=0.9\linewidth]{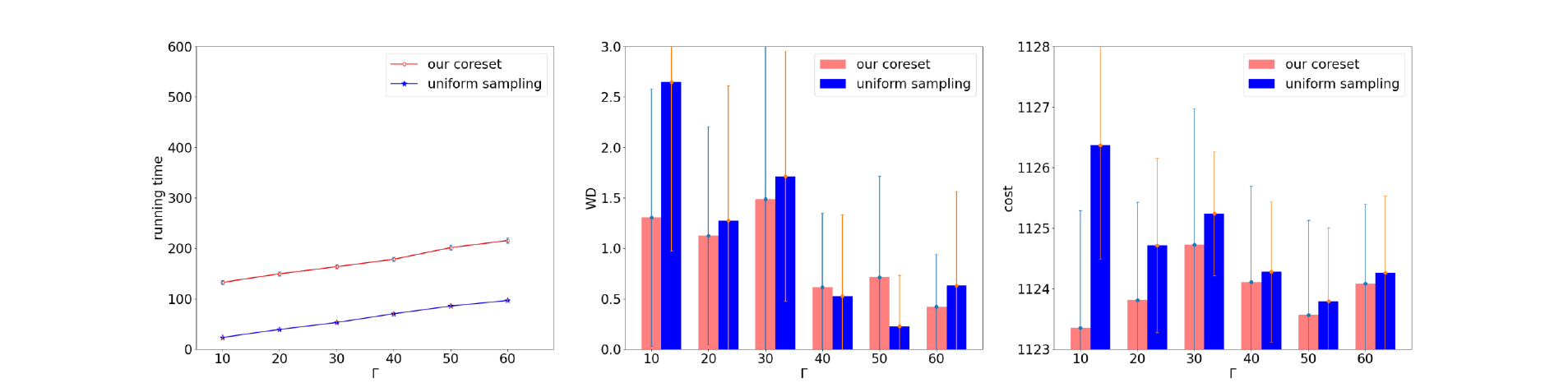}
    \caption{Comparisons of our coreset and uniform sampling on ModelNet40. Shift locations of 60 CAD models according to distribution $\mathcal{N}(0,50)$.}
\label{Fig:MN_sampling_1_50_60}
\end{figure}

\begin{figure}[htbp!]
    \centering
    \includegraphics[width=0.9\linewidth]{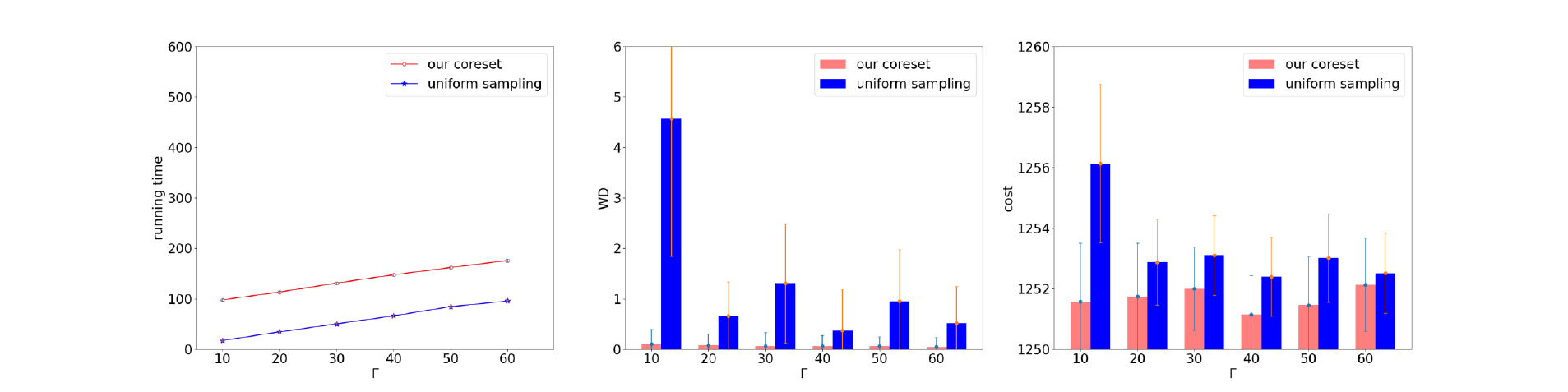}
    \caption{Comparisons of our coreset and uniform sampling on ModelNet40. Shift locations of 90 CAD models according to distribution $\mathcal{N}(0,50)$.}
\label{Fig:MN_sampling_1_50_90}
\end{figure}

\begin{figure}[htbp!]
    \centering
    \includegraphics[width=0.9\linewidth]{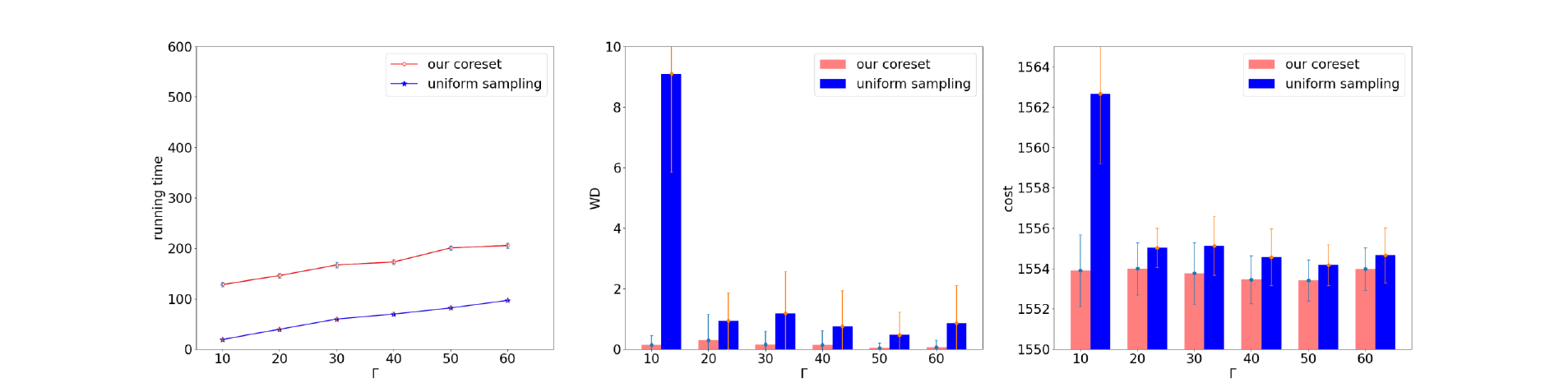}
    \caption{Comparisons of our coreset and uniform sampling on ModelNet40. Shift locations of 30 CAD models according to distribution $\mathcal{N}(0,100)$.}
\label{Fig:MN_sampling_1_100_30}
\end{figure}

\begin{figure}[htbp!]
    \centering
    \includegraphics[width=0.9\linewidth]{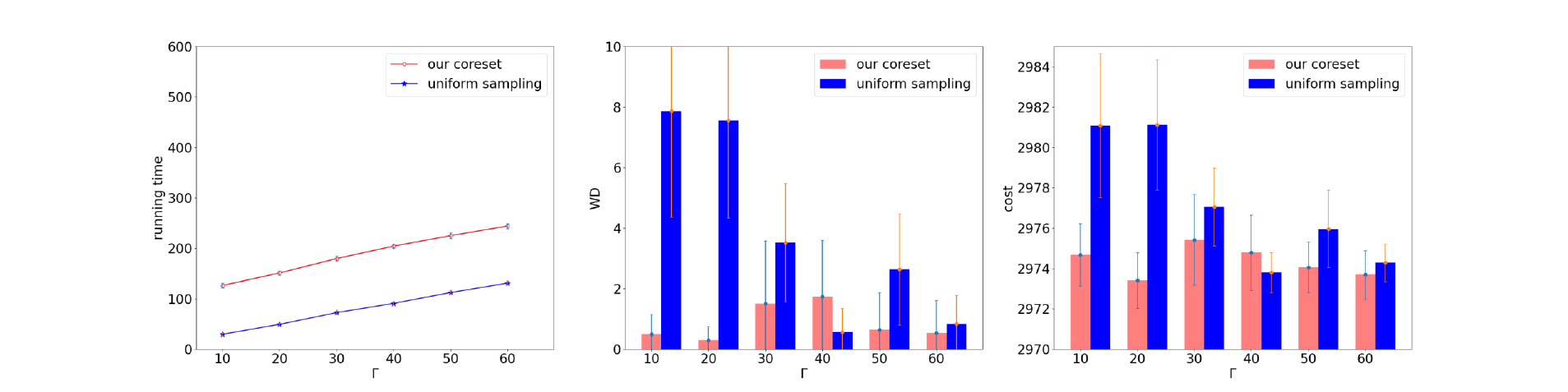}
    \caption{Comparisons of our coreset and uniform sampling on ModelNet40. Shift locations of 60 CAD models according to distribution $\mathcal{N}(0,100)$.}
\label{Fig:MN_sampling_1_100_60}
\end{figure}

\begin{figure}[htbp!]
    \centering
    \includegraphics[width=0.9\linewidth]{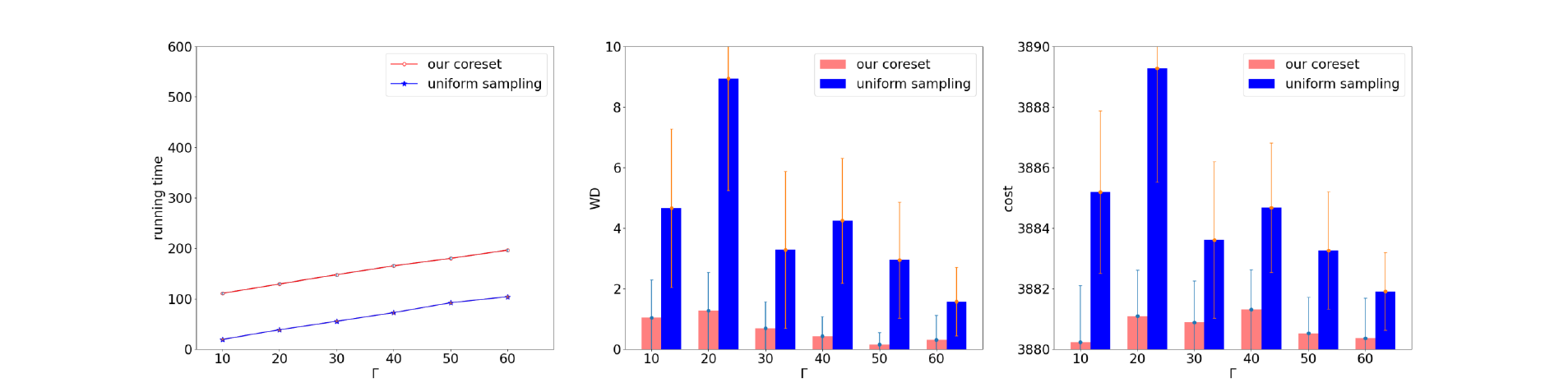}
    \caption{Comparisons of our coreset and uniform sampling on ModelNet40. Shift locations of 90 CAD models according to distribution $\mathcal{N}(0,100)$.}
\label{Fig:MN_sampling_1_100_90}
\end{figure}

\begin{figure}[htbp!]
    \centering
    \includegraphics[width=0.9\linewidth]{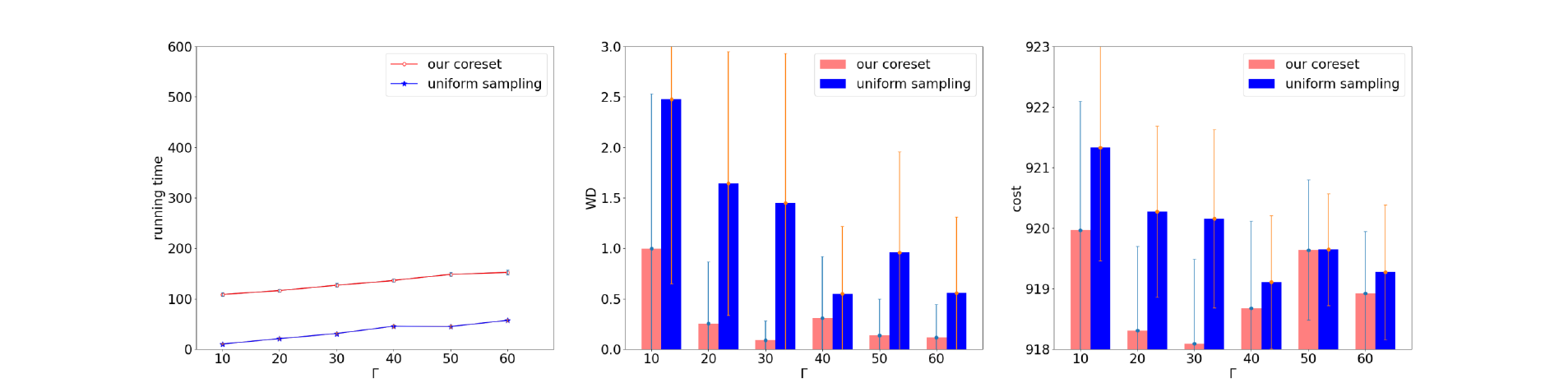}
    \caption{Comparisons of our coreset and uniform sampling on ModelNet40. Shift locations of 30 CAD models according to distribution $\mathcal{N}(0,150)$.}
\label{Fig:MN_sampling_150_30}
\end{figure}

\begin{figure}[htbp!]
    \centering
    \includegraphics[width=0.9\linewidth]{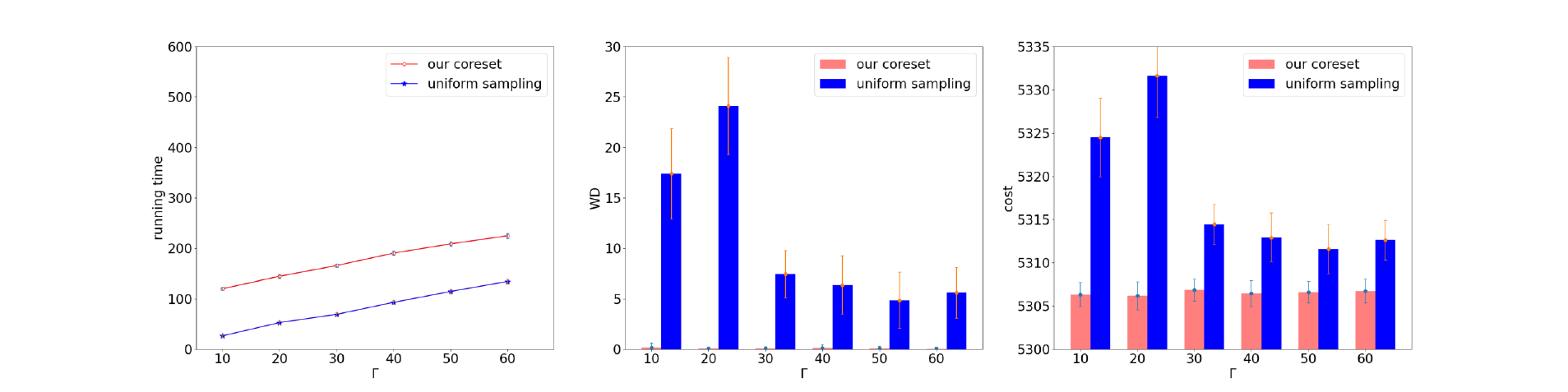}
    \caption{Comparisons of our coreset and uniform sampling on ModelNet40. Shift locations of 60 CAD models according to distribution $\mathcal{N}(0,150)$.}
\label{Fig:MN_sampling_1_150_60}
\end{figure}

\begin{figure}[htbp!]
    \centering
    \includegraphics[width=0.9\linewidth]{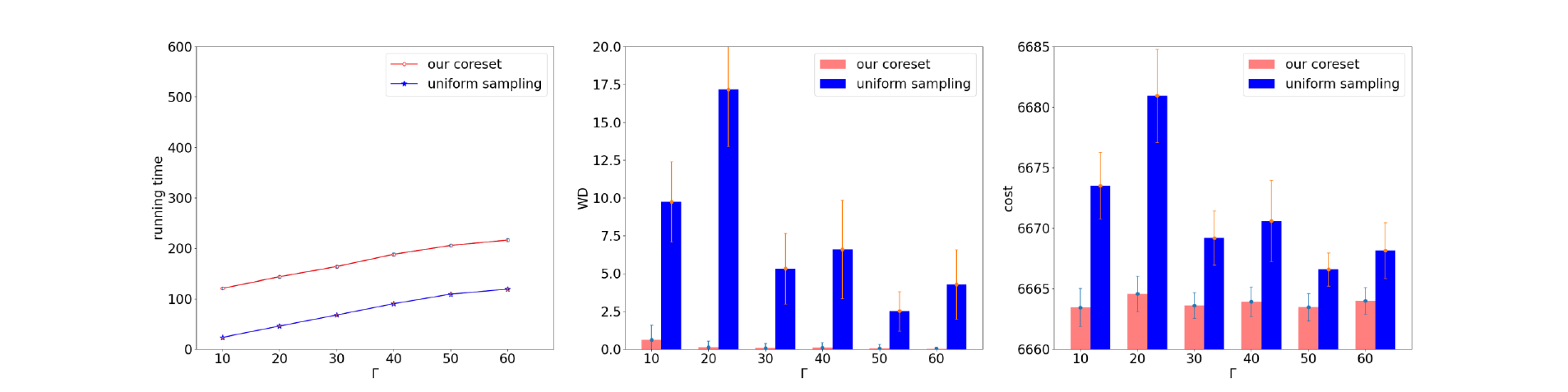}
    \caption{Comparisons of our coreset and uniform sampling on ModelNet40. Shift locations of 90 CAD models according to distribution $\mathcal{N}(0,150)$.}
\label{Fig:MN_sampling_1_150_90}
\end{figure}

\wxWhite{11111111111111111111}\\

\vspace{10mm}

\subsection{ Numerical Instability of  Le et al. \cite{le2021robust}}
\label{App:numerical}

In this section,  we perform experiments on MNIST dataset to illustrate the numerical instability of  Le et al. \cite{le2021robust}. 
We select $300$ images and  obtain a clean dataset $\wxcalQ$.
We add $\zeta = 0.2$ mass of Gaussian noise from $\mathcal{N}(40,\cdot)$ to each measure $\mu^l\in\wxcalQ$ to obtain a noisy dataset $\wxcalQ'$. Then, we compute barycenter by the original fixed-support WB algorithm and our \wxfixRWB \ on the noisy dataset respectively. 
The results in the following table show that the method in \cite{le2021robust} suffers from 
numerical instability. The symbol $\#$ denotes that we cannot compute the result due to numerical instability problem.

\begin{longtable}{llcll}
    \toprule
    N.D. & method & runtime ($\downarrow$) & \textsf{WD} ($\downarrow$) & \textsf{cost} ($\downarrow$)  \\
    \midrule 
    \endfirsthead   
   \multicolumn{3}{c}%
   {{\bfseries \wxWhite{111111111111} -- continued from previous page}} \\
    \toprule
    N.D. & method & runtime ($\downarrow$) & \textsf{WD} ($\downarrow$) & \textsf{cost} ($\downarrow$)  \\
    \midrule
    \endhead 
   \hline \multicolumn{3}{l}{{Continued on next page}} \\ 
   \endfoot  
   \hline
   \endlastfoot
\multirow{8}*{$\mathcal{N}(40,0^2)$} & our \wxfixRWB         & 5.06 & 178.99 & 181.66 \\                                 
                                     & fixed-support WB      & 5.23 & 179.62 & 182.29 \\                                      
                                     & UOT ($\gamma = 1$)    & 2.34 & 332.58 & 331.17 \\                                  
                                     & UOT ($\gamma = 4$)    & 2.34 & 246.93 & 248.12 \\                            
                                     & UOT ($\gamma = 16$)   & 2.34 & 198.82 & 199.93 \\                          
                                     & UOT ($\gamma = 64$)   & 2.33 & 181.89 & 182.59 \\                          
                                     & UOT ($\gamma = 256$)  & 2.33 & 177.15 & 177.75 \\                          
                                     & UOT ($\gamma = 1024$) & 2.51 & 175.93 & 176.50 \\   \hline                       
\multirow{8}*{$\mathcal{N}(40,5^2)$} & our \wxfixRWB         & 4.61 & 170.53 & 181.79  \\                                
                                     & fixed-support WB      & 4.24 & 172.74 & 184.04  \\                                     
                                     & UOT ($\gamma = 1$)    & 2.29 & 6.55   & 9.49  \\                                  
                                     & UOT ($\gamma = 4$)    & 2.17 & 32.63  & 37.17  \\                            
                                     & UOT ($\gamma = 16$)   & 2.18 & 201.32 & 212.21  \\                          
                                     & UOT ($\gamma = 64$)   & 2.25 & 379.00 & 390.25  \\                          
                                     & UOT ($\gamma = 256$)  & 2.23 & 452.12 & 463.62  \\                          
                                     & UOT ($\gamma = 1024$) & 2.22 & 474.26 & 485.96  \\   \hline      
\multirow{8}*{$\mathcal{N}(40,10^2)$} & our \wxfixRWB        & 5.19 & 132.14 & 136.37    \\                                 
                                      & fixed-support WB     & 3.42 & 157.34 & 161.61    \\                              
                                      & UOT ($\gamma = 1$)   &   \#  &   \#    &    \#      \\                                  
                                      & UOT ($\gamma = 4$)   &   \#  &   \#    &    \#      \\                            
                                      & UOT ($\gamma = 16$)  &   \#  &   \#    &    \#      \\                          
                                      & UOT ($\gamma = 64$)  &   \#  &   \#    &    \#      \\                          
                                      & UOT ($\gamma = 256$) &   \#  &   \#    &    \#      \\                          
                                      & UOT ($\gamma = 1024$)&   \#  &   \#    &    \#      \\   \hline   
\multirow{8}*{$\mathcal{N}(40,15^2)$} & our \wxfixRWB        & 6.81 & 60.66  & 72.76  \\                             
                                      & fixed-support WB     & 3.51 & 124.52 & 139.82  \\                                       
                                      & UOT ($\gamma = 1$)   &  \#   &   \#    &   \#    \\                                  
                                      & UOT ($\gamma = 4$)   &  \#   &   \#    &   \#    \\                            
                                      & UOT ($\gamma = 16$)  &  \#   &   \#    &   \#    \\                          
                                      & UOT ($\gamma = 64$)  &  \#   &   \#    &   \#    \\                          
                                      & UOT ($\gamma = 256$) &  \#   &   \#    &   \#    \\                          
                                      & UOT ($\gamma = 1024$)&  \#   &   \#    &   \#    \\   \hline   
\multirow{8}*{$\mathcal{N}(40,20^2)$} & our \wxfixRWB        &  7.57 & 48.39   & 50.94    \\                                   
                                      & fixed-support WB     &  3.66 & 130.31  & 129.04   \\                                   
                                      & UOT ($\gamma = 1$)   &   \#  &    \#   &     \#  \\                        
                                      & UOT ($\gamma = 4$)   &   \#  &    \#   &     \#  \\                  
                                      & UOT ($\gamma = 16$)  &   \#  &    \#   &     \#  \\                
                                      & UOT ($\gamma = 64$)  &   \#  &    \#   &     \#  \\                
                                      & UOT ($\gamma = 256$) &   \#  &    \#   &     \#  \\                
                                      & UOT ($\gamma = 1024$)&   \#  &    \#   &     \#  \\  \hline
\multirow{8}*{$\mathcal{N}(40,25^2)$}& our \wxfixRWB         & 4.18 & 39.79  & 42.23   \\                                      
                                     & fixed-support WB      & 4.37 & 114.63 & 116.23  \\                                   
                                     & UOT ($\gamma = 1$)    &  \#   &   \#    &   \#  \\                                  
                                     & UOT ($\gamma = 4$)    &  \#   &   \#    &   \#  \\                            
                                     & UOT ($\gamma = 16$)   &  \#   &   \#    &   \#  \\                          
                                     & UOT ($\gamma = 64$)   &  \#   &   \#    &   \#  \\                          
                                     & UOT ($\gamma = 256$)  &  \#   &   \#    &   \#  \\                          
                                     & UOT ($\gamma = 1024$) &  \#   &   \#    &   \#  \\   \hline                                         
\multirow{8}*{$\mathcal{N}(40,30^2)$} & our \wxfixRWB        & 7.59  &  35.06  & 35.82   \\                                       
                                     & fixed-support WB      & 8.75  &  127.58 & 112.89 \\                                       
                                     & UOT ($\gamma = 1$)    &  \#   &    \#   &   \#    \\                                  
                                     & UOT ($\gamma = 4$)    &  \#   &    \#   &   \#    \\                            
                                     & UOT ($\gamma = 16$)   &  \#   &    \#   &   \#    \\                          
                                     & UOT ($\gamma = 64$)   &  \#   &    \#   &   \#    \\                          
                                     & UOT ($\gamma = 256$)  &  \#   &    \#   &   \#    \\                          
                                     & UOT ($\gamma = 1024$) &  \#   &    \#   &   \#    \\   \hline   
\multirow{8}*{$\mathcal{N}(40,35^2)$} & our \wxfixRWB         & 8.83 & 20.71   & 24.00    \\                                           
                                      & fixed-support WB      & 6.61 & 100.41  & 103.10    \\                                           
                                      & UOT ($\gamma = 1$)    &  \#  &    \#   &     \#  \\                                  
                                      & UOT ($\gamma = 4$)    &  \#  &    \#   &     \#  \\                            
                                      & UOT ($\gamma = 16$)   &  \#  &    \#   &     \#  \\                          
                                      & UOT ($\gamma = 64$)   &  \#  &    \#   &     \#  \\                          
                                      & UOT ($\gamma = 256$)  &  \#  &    \#   &     \#  \\                          
                                      & UOT ($\gamma = 1024$) &  \#  &    \#   &     \#  \\   \hline   
\multirow{8}*{$\mathcal{N}(40,40^2)$} & our \wxfixRWB         & 5.13 & 12.55   & 16.49   \\                                      
                                      & fixed-support WB      & 6.22 & 108.43  & 104.04  \\                                      
                                      & UOT ($\gamma = 1$)    &   \# &    \#   &      \# \\                                  
                                      & UOT ($\gamma = 4$)    &   \# &    \#   &      \# \\                            
                                      & UOT ($\gamma = 16$)   &   \# &    \#   &      \# \\                          
                                      & UOT ($\gamma = 64$)   &   \# &    \#   &      \# \\                          
                                      & UOT ($\gamma = 256$)  &   \# &    \#   &      \# \\                          
                                      & UOT ($\gamma = 1024$) &   \# &    \#   &      \# \\                                           
\end{longtable}

\section{Other preliminaries}

\begin{definition}[Doubling dimension \cite{huang2018epsilon}]
 The doubling dimension of metric space $(\wxcalX,\wxdist)$ is the least $\wxDdim$ such that every ball with radius $2r$ can be covered by at most $2^\wxDdim$ balls with radius $r$.
\end{definition}

Roughly speaking, the doubling dimension is a measure for describing the growth rate of the data set $\wxcalX$ with respect to the metric $\wxdist(\cdot,\cdot)$. As a special case, the doubling dimension of the Euclidean space $\wxbbR^d$ is $\Theta(d)$.

\begin{definition}[Wasserstein barycenter]     \label{Def:WB}
Given a set of probability measures $\wxcalQ$ as in \wxRefEq{Eq:Q},
    the Wasserstein barycenter on
    $\wxcalQ$ is a new probability measure $\nu^*\in\wxcalPX$ that minimizes the following objective    
 
    \begin{equation*}
        WB(\wxcalQ,\nu) := \frac{1}{\omega(\wxcalQ)}\sum_{l=1}^m \omega_l \cdot W_z^z(\mu^l,\nu) \text{ with } \nu\in\wxcalPX.
    \end{equation*}
\end{definition}

\begin{lemma}
    \label{Lem:Generalized triangle inequality}
    Given three points $a,b,c \in \wxcalX$ and a real number $z\geq 1$, the following Generalized triangle inequalities\cite{makarychev2022performance,sohler2018strong} hold for any $0<s\leq 1$:
    \begin{itemize}
        \item $\wxdist^z(a,b) \leq (1+s)^{z-1} \wxdist^z(a,c) + (1+\frac{1}{s})^{z-1} \wxdist^z(b,c)$;
        
        \item $\wxVert{\wxdist^z(a,c) - \wxdist^z(b,c)}    \leq s\cdot \wxdist^z(a,c) + (\frac{3z}{s})^{z-1}\wxdist^z(a,b)$. 
        
    \end{itemize}
\end{lemma}

\section{Proof of \Cref{Sec:model}}
\label{App:proof of Sec3}

\begin{proof}(\textbf{of \Cref{Lem:reduction AOT}})
\textbf{Equivalence of $\wxcalD_g$ and $\wxcalD_h$ under $\varphi$:} First, we prove $\varphi:\wxcalD_g\rightarrow\wxcalD_h$ is a map. That is, given any $(\wxbfaout^0,\wxbfbout^0,\wxbfP^0) \in \wxcalD_g$, we need to prove $\wxbfPAug^0 = \varphi(\wxbfaout^0,\wxbfbout^0,\wxbfP^0)$ is the only image in $\wxcalD_h$ as follows. 
From \wxRefEq{Eq:phi}, for any $j\in[n]$, we have 

\begin{equation*}
 \begin{aligned}
  \sum_{i=1}^{n+1}(\wxPAug^0)_{ij} = \sum_{i=1}^n P^0_{ij} + \frac{(\wxbout^0)_j}{1-\zeta_{\nu}} 
 = \frac{b^0_j - (\wxbout^0)_j}{1-\zeta_{\nu}} + \frac{(\wxbout^0)_j}{1-\zeta_{\nu}} = \frac{b^0_j}{1-\zeta_{\nu}} = (\wxbfbAug^0)_j,
 \end{aligned}
\end{equation*}
where the second equality comes from the fact that $(\wxbfP^0)^T\wxbfOne = \frac{\wxbfb^0 - \wxbfbout^0}{1-\zeta_{\nu}}$ in \wxRefEq{Eq:RWD}, and the last equality is due to \wxRefEq{Eq:notation for AOT}. Similarly, by \wxRefEq{Eq:RWD} and \wxRefEq{Eq:notation for AOT}, we obtain $\sum_{i=1}^{n+1}(\wxPAug^0)_{ij} = \frac{\zeta_{\mu}}{1-\zeta_{\mu}} = (\wxbfbAug^0)_j$ for $j = n+1$.
Till now, we have $(\wxbfPAug^0)^T\wxbfOne = \wxbfbAug^0$. By a similar way, we have $\wxbfPAug^0\wxbfOne = \wxbfaAug^0$. Besides, it is obvious that $\wxbfPAug^0\in\wxbbR^{(n+1)\times (n+1)}$. Therefore, for each $(\wxbfaout^0,\wxbfbout^0,\wxbfP^0)\in\wxcalD_g$, we have the only $\wxbfPAug^0 = \varphi(\wxbfaout^0,\wxbfbout^0,\wxbfP^0) \in\wxcalD_h$, which indicates that $\varphi$ is a map. 
Then, it is obvious that $ \varphi(\wxbfaout^0,\wxbfbout^0,\wxbfP^0) = \varphi(\wxbfaout',\wxbfbout',\wxbfP')$ holds only if $ (\wxbfaout^0,\wxbfbout^0,\wxbfP^0) = (\wxbfaout',\wxbfbout',\wxbfP')$; thus, $\varphi$ is a injection.

Next, to prove $\varphi$ is a surjection, we construct $$\psi:\wxcalD_h\rightarrow\wxcalD_g, \wxbfPAug^1\mapsto (\wxbfaout^1,\wxbfbout^1,\wxbfP^1),$$ where 
$P^1_{ij} = (\wxPAug^1)_{ij}, 
		(\wxaout^1)_i = (1-\zeta_{\mu}) \cdot (\wxPAug^1)_{i,n+1},
		(\wxbout^1)_j = (1-\zeta_{\nu}) \cdot (\wxPAug^1)_{(n+1,j)}$
for $i,j\in[n]$. We can demonstrate that $\psi$ is also a map by similar techniques.

Then, given any 
$\wxbfPAug^1\in\wxcalD_h$, we can construct $(\wxbfaout^1,\wxbfbout^1,\wxbfP^1) = \psi(\wxbfPAug^1) \in\wxcalD_g$; besides, we have $\varphi(\wxbfaout^1,\wxbfbout^1,\wxbfP^1) = \wxbfPAug^1$. That is, for any $\wxbfPAug^1\in\wxcalD_h$, we can find its preimage in $\wxcalD_g$. Thus, $\varphi$ is a surjective.
Therefore, $\varphi$ is a bijection between $\wxcalD_g$ and $\wxcalD_h$.

\noindent \textbf{Equivalence of $h\varphi$ and $g$:} 

For any $(\wxbfaout,\wxbfbout,\wxbfP)\in\wxcalD_g$, we can construct $\wxbfPAug = \varphi(\wxbfaout,\wxbfbout,\wxbfP)\in\wxcalD_h$. Then, we have
    \begin{equation*}
        \begin{aligned}
            g(\wxbfaout,\wxbfbout,\wxbfP) = \langle \wxbfP,\wxbfC \rangle = \sum_{i=1}^{n}\sum_{j=1}^{n} P_{ij} \cdot C_{ij} = 
 \sum_{i=1}^{n+1}\sum_{j=1}^{n+1} (\wxPAug)_{ij} \cdot (\wxCAug)_{ij} = 
 \langle \wxbfPAug,\wxbfCAug\rangle = h(\wxbfPAug) = h \varphi(\wxbfaout,\wxbfbout,\wxbfP)
        \end{aligned}
    \end{equation*}
where the third equality comes from \wxRefEq{Eq:AOT} and the definition of $\varphi$. Specially, from \wxRefEq{Eq:AOT}, we have $(\wxCAug)_{ij} = 0$ if $i=n+1$ or $j=n+1$, and $(\wxCAug)_{ij} = C_{ij}$ for $i,j\in[n]$; via the definition of $\varphi$, we have $P_{ij} = (\wxPAug)_{ij}$ for $i,j\in[n]$.
Thus, we obtain $g = h\varphi$. 

Similarly, for any $\wxbfPAug\in\wxcalD_h$, we can also prove $h(\wxbfPAug) = g\varphi^{-1}(\wxbfPAug)$, where $\varphi^{-1} := \psi$. Thus, we have $h = g\varphi^{-1}$. 
\end{proof}

\begin{proof}(\textbf{of \Cref{Them:RWB}})
For any $\mu,\nu\in\wxcalPX$, by combining \wxRefEq{Eq:RWD},\wxRefEq{Eq:AOT} and \Cref{Lem:reduction AOT}, we have $\wxcalWzz(\mu,\nu) = \wxcalAOT(\mu,\nu)$. Then,

 \begin{equation*}
 \begin{aligned}
 \wxcalRWB(\wxcalQ,\nu) = \frac{1}{\omega(\wxcalQ)} \sum_{l=1}^m \omega_l\cdot \wxcalWzz(\mu^l,\nu) = \frac{1}{\omega(\wxcalQ)} \sum_{l=1}^m \omega_l\cdot \wxcalAOT(\mu^l,\nu) = \wxcalAWB(\wxcalQ,\nu).
 \end{aligned}
 \end{equation*}
\end{proof}

\section{Proof of \Cref{Sec:free}}
\label{App:proof of Sec4}

\begin{lemma} \label{Lem:approx of WB}
    If we take a sample $\mu$ from $\wxcalQ$ according to the distribution $\frac{\omega_l}{\omega(\wxcalQ)}$, 
    then for any $\nu\in\wxcalPX$ we have
    \begin{equation*}
        \wxbbE\wxmiddleBracket{ \sum_{l=1}^m\omega_l\cdot W_z^z(\mu^l,\mu) }\leq 2^z \cdot \sum_{l=1}^m\omega_l\cdot W_z^z(\mu^l,\nu)
    \end{equation*}
\end{lemma}

\begin{proof}[\textbf{of \Cref{Lem:approx of WB}}]
From the generalized triangle inequality in \Cref{Lem:Generalized triangle inequality}, we have
\begin{equation}
\begin{aligned}
& \wxbbE\wxmiddleBracket{ \sum_{l=1}^m\omega_l\cdot W_z^z(\mu^l,\mu) } \\
\leq & \wxbbE\wxmiddleBracket{ \sum_{l=1}^m\omega_l\cdot2^{z-1}\cdot (W_z^z(\mu^l,\nu) + W_z^z(\mu,\nu)) } \label{Subeq:002}
\end{aligned}
\end{equation}

 Since the expectation is taken on $\mu$, we have 
 \begin{equation*}
 \begin{aligned}
 & \wxRefEq{Subeq:002} = \sum_{l=1}^m\omega_l\cdot2^{z-1}\cdot W_z^z(\mu^l,\nu) + \sum_{l=1}^m\omega_l\cdot2^{z-1}\cdot\wxbbE\wxmiddleBracket{ W_z^z(\mu,\nu) } \\
= & \sum_{l=1}^m\omega_l\cdot2^{z-1}\cdot W_z^z(\mu^l,\nu) + 2^{z-1} \cdot \omega(\wxcalQ) \cdot\wxbbE\wxmiddleBracket{ W_z^z(\mu,\nu) } \label{Subeq:003} \\
= & 2^z \cdot \sum_{l=1}^m\omega_l\cdot W_z^z(\mu^l,\nu).
\end{aligned}
\end{equation*}
\end{proof}

\begin{proof}(\textbf{of \Cref{Lem:approx solution}})
W.L.O.G., we assume that $\widehat{\wxcalQ} = \wxbigBracket{\hat{\mu}^l}_{l\in[m]}$ is the optimal noiseless measure set induced by the optimal solution $\nu^*$ for \wxfreeRWB \ on $\wxcalQ$;
that is,
\begin{equation*}
 \begin{aligned}
 \wxOPT = \wxcalRWB(\wxcalQ,\nu^*) = WB(\widehat{\wxcalQ},\nu^*) = \frac{1}{\omega(\wxcalQ)}\cdot\sum_{l=1}^m \omega_l \cdot W_z^z(\hat{\mu}^l,\nu^*).
 \end{aligned}
\end{equation*}
It is obvious that $\wxOPT$ and $\nu^*$ are also optimal value and optimal solution for Wasserstein barycenter problem on $\widehat{\wxcalQ}$.
By combining \Cref{Lem:approx of WB} and Markov's inequality, if we take $t$ samples $\wxbigBracket{ \hat{\mu}^{q1},\ldots,\hat{\mu}^{qt} }$ from $\widehat{\wxcalQ}$, then there exists an $2^z\alpha$-approximate solution $\hat{\mu}$ in $\wxbigBracket{ \hat{\mu}^{q1},\ldots,\hat{\mu}^{qt} }$ for Wasserstein barycenter problem on $\widehat{\wxcalQ}$ with probability at least $1-\alpha^{-t}$. 

Then, we have 
\begin{equation*}
 \begin{aligned}
 \wxcalRWB(\wxcalQ,\hat{\mu}) \leq WB(\widehat{\wxcalQ},\hat{\mu}) \leq 2^z\alpha\cdot WB(\widehat{\wxcalQ}, \nu^*)
 = 2^z\alpha \cdot \wxOPT,
 \end{aligned}
\end{equation*}
where the first inequality comes from that $\widehat{\wxcalQ}$ is the optimal noiseless measure set induced by $\nu^*$ instead of $\hat{\mu}$. ($\widehat{\wxcalQ}$ is the feasible noiseless measure set for $\hat{\mu}$, and is necessary the optimal one.)
Thus, $\hat{\mu}$ is a $2^z\alpha$-approximate solution for \wxfreeRWB \ on $\wxcalQ$.

However, $\hat{\mu}$ is unknown for us. Suppose $\widehat{\wxbfP}^l$ is the optimal coupling induced by Wasserstein distance between $\hat{\mu}$ and $\hat{\mu}^l$. 
Then, $\wxbigBracket{\widehat{\wxbfP}^l}_{l\in[m]}$ is a feasible coupling set for \wxfixRWB\ $\wxcalRWB_{\widehat{X}}(\wxcalQ,\hat{\mu})$, where $\widehat{X}$ is the locations of $\hat{\mu}$. 
Suppose $\hat{\nu} = \min_{\nu}\wxcalRWB_{\widehat{X}}(\wxcalQ,\nu)$, we have $\wxcalRWB_{\widehat{X}}(\wxcalQ,\hat{\nu}) \leq \wxcalRWB_{\widehat{X}}(\wxcalQ,\hat{\mu})$.
For the fixed $\hat{\mu}$ and $\hat{\nu}$, their locations are both $\widehat{X}$, thus 
$\wxcalRWB_{\widehat{X}}(\wxcalQ,\hat{\mu}) = \wxcalRWB(\wxcalQ,\hat{\mu})$ and $\wxcalRWB_{\widehat{X}}(\wxcalQ,\hat{\nu}) = \wxcalRWB(\wxcalQ,\hat{\nu})$. 

Then, we have
\begin{equation*}
 \wxcalRWB(\wxcalQ,\hat{\nu}) \leq \wxcalRWB(\wxcalQ,\hat{\mu}) \leq 2^z\alpha\cdot \wxOPT,
\end{equation*}
which implies that $\hat{\nu}$ is a $2^z\alpha$-approximate solution for \wxfreeRWB \ on $\wxcalQ$. Finally, according to the definition of $\wxnuTilde$, $\wxnuTilde$ must be an $2^z\alpha$-approximate solution.
\end{proof}

\begin{proof}(\textbf{of \Cref{Lem:bounded by r}})
    Let $\mu^0,\mu^1$ be the noiseless probability measures induced by $\wxcalRWDz(\mu,\nu)$ and $\wxcalRWDz(\mu,\wxnuTilde)$ respectively; that is, $\wxcalRWDz(\mu,\nu) = W_z(\mu^0,\nu)$ and $\wxcalRWDz(\mu,\wxnuTilde) = W_z(\mu^1,\wxnuTilde)$. Then, we have
    
    \begin{equation*}
        \wxcalRWDz^z(\mu,\nu) = W_z^z(\mu^0,\nu) \\
        \leq W_z^z(\mu^1,\nu) 
    \end{equation*}
    By using the generalized triangle inequality in \Cref{Lem:Generalized triangle inequality}, we obtain
    \begin{equation*}
    	\begin{aligned}
        & \wxcalRWDz^z(\mu,\nu) \leq (1+s)^{z-1} \cdot W_z^z(\mu^1,\wxnuTilde) + (1+\frac{1}{s})^{z-1} \cdot W_z^z(\nu,\wxnuTilde) \\
        \leq & (1+s)^{z-1} \cdot \wxcalRWDz^z(\mu,\wxnuTilde) + (1+\frac{1}{s})^{z-1} \cdot r^z,
        \end{aligned}
    \end{equation*}
 where the last inequality comes from the definition of $\wxcalD_{\wxnuTilde}$.
\end{proof}

\begin{proof}(\textbf{of \Cref{Lem:fix a solution}})
According to Hoeffding's inequality, if we set $\Gamma = \frac{\log1/\eta}{\epsilon^2} $, then the following inequality holds with probability at least $1-\eta$ for $0\leq k\leq K$.
 \begin{equation}\label{Eq:003}
 \wxVert{ \sum_{\mu\in\wxcalS_k} \tau(\mu) \cdot \wxcalRWDz^z(\mu,\nu) - \sum_{\mu\in\wxcalQ_k} \omega(\mu) \cdot \wxcalRWDz^z(\mu,\nu) } \leq \wxcalO(\epsilon) \cdot \omega(\wxcalQ_k) \cdot (2^k H)^z
 \end{equation}

We set $r=H$ in \Cref{Def:coreset}, then we have $W_z(\wxnuTilde,\nu) \leq H$ for $\nu\in\wxcalD_{\wxnuTilde}$. 
For any $\mu$ in the outermost layer $\wxcalQ_{K+1}$, by combining generalized triangle inequality and the relationship in \wxRefEq{Eq:partition layers}, we have $\wxcalRWDz^z(\mu,\nu) \in (1+\wxcalO(\epsilon)) \cdot \wxcalRWDz^z(\mu,\wxnuTilde)$ for any $\nu\in\wxcalD_{\wxnuTilde}$. 
Then, by combining \wxRefEq{Eq:muMaxMin_weights}, we have 
 \wxRefEq{Eq:003} holds for $k = K+1$.

Then, we can obtain
\begin{equation*}
 \begin{aligned}
 \wxVert{ \wxcalRWB(\wxcalQ, \nu) - \wxcalRWB(\wxcalS,\nu) } = \frac{1}{\omega(\wxcalQ)} \wxVert{ \sum_{\mu\in\wxcalQ} \omega(\mu) \cdot\wxcalRWDz^z(\mu,\nu) - \sum_{\mu\in\wxcalS} \tau(\mu) \cdot \wxcalRWDz^z(\mu,\nu) } \\
 \leq \frac{1}{\omega(\wxcalQ)} \sum_{k=0}^{K+1} \wxVert{ \sum_{\mu\in\wxcalQ_k} \omega(\mu) \cdot\wxcalRWDz^z(\mu,\nu) - \sum_{\mu\in\wxcalS_k} \tau(\mu) \cdot \wxcalRWDz^z(\mu,\nu) } \\
 \leq \frac{\wxcalO(\epsilon)}{\omega(\wxcalQ)} \cdot \sum_{k=0}^{K+1} \omega(\wxcalQ_k) \cdot (2^k H)^z = \wxcalO(\epsilon)\cdot\wxcalRWB(\wxcalQ,\wxnuTilde) \\
 \end{aligned}
\end{equation*}
where the first inequality follows from the partition in \wxRefEq{Eq:partition layers}, the second inequality comes from \wxRefEq{Eq:003}, and the last inequality is due to that the partition \wxRefEq{Eq:partition layers} is anchored at $\wxnuTilde$.

Since $\wxnuTilde$ is an $\wxcalO(1)$-approximate solution, we have $\wxcalRWB(\wxcalQ, \wxnuTilde)\leq \wxcalO(1)\cdot \wxOPT$. Moreover, we know $\wxOPT \leq \wxcalRWB(\wxcalQ, \nu)$, thus $ \wxVert{ \wxcalRWB(\wxcalQ, \nu) - \wxcalRWB(\wxcalS,\nu) } \leq \wxcalO(\epsilon) \cdot \wxcalRWB(\wxcalQ, \nu)$ holds.
\end{proof}

\bibliographystyle{siam}

\end{document}